\documentclass[10pt]{article}
\usepackage{amsmath,amsthm,amssymb,latexsym,mathtools}
\usepackage{graphicx}
\usepackage{stmaryrd}
\usepackage{etoolbox}

\allowdisplaybreaks[4]

\usepackage{amsmath,amsthm,amsfonts,etoolbox,stmaryrd,enumitem}
\usepackage{booktabs}
\usepackage[shortcuts]{extdash}
\usepackage{algorithm}
\usepackage[noend]{algpseudocode}
\algrenewcommand\algorithmicthen{\relax}
\algdef{C}[IF]{IF}{ElsIf}[1]{\textbf{elif}\ #1}
\algrenewcommand\algorithmicdo{\relax}

\usepackage[authoryear]{natbib}
\bibpunct{(}{)}{;}{a}{,}{,}
\usepackage[pdfpagemode=UseNone,pdfstartview=FitH]{hyperref}

\newtheorem{theorem}{Theorem}

\newtheorem{lemma}[theorem]{Lemma}

\theoremstyle{definition}

\theoremstyle{remark}
\newtheorem{remark}[theorem]{Remark}

\newcommand{\N}{\mathbb{N}}

\newcommand{\R}{\mathbb{R}}

\newcommand{\e}{\mathrm{e}}

\DeclareMathOperator{\E}{\mathbb{E}}
\newcommand{\FFF}{\mathcal{F}}
\newcommand{\NNN}{\mathcal{N}}

\title{Inductive Venn--Abers and related regressors}
\author{Ivan Petej and Vladimir Vovk}

\begin{document}
\maketitle

\begin{abstract}
  Venn--Abers predictors are probabilistic predictors
  that enjoy appealing properties of validity,
  but their major limitation is that they have been applicable
  only to binary classification,
  apart from a recent extension to bounded regression.
  We generalize them to the case of unbounded regression,
  which requires adding an element of conformal prediction.
  In our simulation and empirical studies we investigate
  the predictive efficiency of point regressors derived from Venn--Abers regressors
  and argue that they somewhat improve the predictive efficiency
  of standard regressors for larger training sets.

  \medskip

  \noindent
  The conference version of this paper is to appear in the \emph{Proceedings of Machine Learning Research},
  volume~329 (COPA 2026).
  This version has been updated as compared with the conference version,
  and the experiments in it are based on a new version of our Python library
  \citep{Petej:2026}.
  While the main conclusions are unchanged,
  the experimental results for our methods (the IVAR and CVAR) are slightly better
  than in the COPA 2026 paper.
\end{abstract}

\section{Introduction}

This paper explores the problem of regression estimation
as presented in, e.g., \citet[Sect.~1.4]{Vapnik:1998}.
In the standard setting of statistical learning,
where we observe an IID sequence of random pairs $(X,Y)$
consisting of objects $X$ and their labels $Y\in\R$,
the (ideal) \emph{regression estimator} (of $Y$ given $X$) maps each object $x$
to the expected value $\E(Y\mid X=x)$ of its label $Y$ conditional on observing $x$.
Vapnik lists regression estimation
as one of three basic statistical problems \citep[Sect.~1.2]{Vapnik:1998}.
We develop ideas of conformal prediction
\citep{Vovk/etal:2022book,Angelopoulos/etal:arXiv2411}
and introduce an algorithm for regression estimation
that satisfies a natural distribution-free notion of validity.

We are interested in the distribution-free setting,
when nothing is known about the distribution generating one pair $(X,Y)$
and we are only given a training sequence of labelled objects
and an unlabelled test object.
In typical cases we cannot hope to find the true regression estimate $\E(Y\mid X=x)$
(even when it is well-defined),
and so we lower the bar in three respects
in our definition of a valid regression estimator.
First, we allow estimators of the form $\E(Y\mid\FFF)$ for some $\sigma$-algebra $\FFF$;
we will express this by saying that our regression estimator is perfectly calibrated.
Ideally, $\FFF$ should be close to the $\sigma$-algebra generated
by the training set and test object.
Second, we allow our algorithms to output intervals
(ideally short ``imprecise regression estimates'')
containing $\E(Y\mid\FFF)$.
And third, we replace $\E(Y\mid\FFF)$ by $\E(Y'\mid\FFF)$,
where $Y'$ is a ``regularized'' version of $Y$ such that $Y'=Y$ with high probability.
(This is the element of conformal prediction that we mentioned in the abstract;
it is required in the absence of bounds on the test label.)

As usual, we consider two principal requirements for our algorithms,
validity (technically, perfect calibration) and efficiency.
Validity (in this technical sense) will be guaranteed for our algorithms,
but to achieve efficiency we will use existing regression algorithms
that we believe to be efficient, although perhaps miscalibrated.
We will develop methods for improving their calibration.
These methods will be adaptations
of the methods used in \citet{Vovk/etal:2015-local}
(and presented in much greater detail in \citealt[Chap.~6]{Vovk/etal:2022book})
in the context of binary classification,
with bounded regression considered earlier by \citet{Laan/Alaa:2024}.

The problem of regression estimation considered in this paper
is very different from the regression problems in conformal prediction,
where the task is to output a prediction region
(typically a prediction interval) for a test label
with a pre-specified coverage probability.
In regression estimation the task is to cover the expected label $\E(Y\mid\FFF)$
rather than the label $Y$ itself.
This will allow us to produce much shorter intervals
(especially as we replace $Y$ by $Y'$ that is not guaranteed to coincide with $Y$).

The main goal of the computational experiments reported in this paper
is to demonstrate that application of our methods leads to an improvement
in the performance of standard point regressors
(we consider those implemented in \texttt{scikit-learn}).
The improvement is not as significant as we had hoped
and disappears for smaller datasets.

\section{Comparisons with Literature}
\label{sec:literature}

The results of this paper are not directly comparable with the existing literature
because of a somewhat unusual notion of validity that we use.
But there are several related approaches.

A strand of research that is closely related
to what we do in this paper is conformal regression,
already mentioned in the previous section.
Important advances in conformal regression include \citet{Romano/etal:2019},
offering very flexible methods based on quantile regression,
and \citet{Gibbs/etal:2025},
establishing conditional validity results.
See, e.g., \citet{Angelopoulos/etal:arXiv2411} for a recent review
of conformal prediction in general and conformal regression in particular.

The goal of conformal regression is to produce provably valid,
under the assumption of IID data,
prediction intervals for the label of a test object.
Since validity here means a guaranteed coverage probability,
this is a much more ambitious problem than regression estimation
dealt with in this paper,
as mentioned earlier.

An even more closely related, but still very different, problem
is that of covering the conditional expectation $\E(Y\mid X=x)$
for a given test object $x$.
Interestingly, for a nonatomic distribution of objects,
this problem is as difficult as predicting the label $Y$ itself
\citep[Theorem~11.3]{Angelopoulos/etal:arXiv2411}.

Conformal predictive distributions \citep[Chap.~7]{Vovk/etal:2022book}
are different from conformal regression in that, instead of prediction intervals,
they output full predictive distributions for future labels
(assumed to be real-valued, as in conformal regression).
To achieve validity these predictive distributions should be imprecise
in a certain sense.
Our current task is easier in that we only aim to cover $\E(Y'\mid\FFF)$,
but we will achieve another property of validity,
namely perfect calibration instead of calibration in probability
(the two notions of calibration are not comparable in general).
See \citet{Allen/etal:arXiv2503} for a recent analysis
of achievable properties of validity for conformal predictive distributions.

The property of validity used in this paper is inherited from Venn prediction,
which is, however, typically applied to produce imprecise probability forecasts
in classification problems.
Venn prediction (including Venn--Abers prediction) is reviewed
in \citet[Chap.~6]{Vovk/etal:2022book}
and Venn--Abers prediction is reviewed
in \citet[Sect.~12.4]{Angelopoulos/etal:arXiv2411}.

Venn--Abers predictors were extended to the case of bounded regression
by \citet{Laan/Alaa:2024}, as mentioned earlier.
Our bounded Venn--Abers regressor introduced in the next section
is just an inductive version of Algorithm~1 in \citet{Laan/Alaa:2024}.
As the next step, van der Laan and Alaa develop an algorithm
\citep[Algorithm~2]{Laan/Alaa:2024}
for ``self-calibrating conformal prediction'' based on their Algorithm 1
(so that bounded Venn--Abers regression is only a small part
of \citealt{Laan/Alaa:2024}).
Finally, \citet{Laan/Alaa:2025} extend these results to general loss functions
interpreting the usual regression setting as the case of squared error loss.

A key step in our main algorithm is replacing the true labels $Y$
by their regularized versions $Y'$, as discussed in the previous section.
We need it to avoid the vacuous regression interval $(-\infty,\infty)$
(see the end of Sect.~\ref{sec:IVAR}).
Assumptions made by \citet{Laan/Alaa:2024} and \citet{Laan/Alaa:2025}
allow them to bypass this step,
but we avoid making any assumptions on the data-generating distribution
apart from the observations being IID.

In general, the properties of validity of our procedures
are sufficiently different from the properties of validity
considered in the literature to make direct comparison between them
in empirical or simulation studies
difficult or impossible.
Therefore, in our experimental section (Sect.~\ref{sec:experiments})
we concentrate on evaluating the predictive efficiency of point regressors
derived from imprecise regressors implementing our methods;
namely, we compare the predictive performance of those point regressors
with that of the base algorithms.

\section{Inductive Venn--Abers Regressors}
\label{sec:IVAR}

In this section we ignore the computational complexity
of our methods (it will be one of the topics of Sect.~\ref{sec:algorithm}).
Fix a measurable space $\mathbf{X}$ (the \emph{object space})
and set $\mathbf{Z}:=\mathbf{X}\times\R$ (this is the \emph{example space}).
Each example $z=(x,y)\in\mathbf{Z}$ consists of an object $x\in\mathbf{X}$
and its real-valued label $y$.

We consider two settings (leading to formally different, albeit similar,
prediction algorithms).
In the setting of \emph{bounded regression}
we are given a finite interval $[C_*,C^*]\subseteq\R$
guaranteed to contain all labels,
training and test.
Otherwise, we have \emph{unbounded regression}.
In the former case, the algorithm and its property of validity
are simpler, and we consider them separately
even though we are not particularly interested in them per se
as we would like to avoid any assumptions apart from IID observations.

Notice that an algorithm for unbounded regression can also be applied
in the bounded setting,
and it may well be preferable if the bounds $C_*$ and $C^*$ are loose.
This is another reason why our main interest is in unbounded regression.

\subsection{Bounded regression}
\label{subsec:bounded}

Fix any regression algorithm (such as a neural net)
producing point predictions as the \emph{base algorithm}.
We are given a training set consisting of $l>k$ examples
(where $k\in\N:=\{1,2,\dots\}$ is typically large
and $l-k\in\N$ is also large)
and a test object $x\in\mathbf{X}$;
$k$ is a parameter of our algorithm.

In bounded regression,
we are given an interval $[C_*,C^*]$ guaranteed to contain all labels.
The corresponding \emph{bounded inductive Venn--Abers regressor} (bounded IVAR)
produces the \emph{regression interval}
\begin{equation}\label{eq:interval}
  [\hat y_*,\hat y^*]
  :=
  \left[
    f_*(r),f^*(r)
  \right]
\end{equation}
for the test label,
where $f^*$, $f_*$, and $r$ are defined as follows:
\begin{enumerate}
\item
  Randomly split the training set of size $l>k$ into two parts,
  a \emph{proper training set} of size $l-k$
  and a \emph{calibration set} $z_1,\dots,z_k$ of size $k$;
  for each $i\in\{1,\dots,k\}$,
  $z_i=(x_i,y_i)$ consists of an object $x_i$ and its label $y_i$.
\item
  Train the regression algorithm on the proper training set
  obtaining a \emph{prediction rule} $R:\mathbf{X}\to\R$
  (a measurable function)
  mapping objects to their predicted labels.
\item
  Find the prediction $r_i:=R(x_i)$ (\emph{base prediction})
  for the calibration object $x_i$,
  $i=1,\dots,k$,
  and the base prediction $r:=R(x)$ for the test object $x$.
\item\label{it:bounded-fit-1}
  Fit isotonic regression to $(r_1,y_1),\dots,(r_k,y_k),(r,C^*)$
  obtaining an isotonic calibrator~$f^*$.
\item
  Set $\hat y^*:=f^*(r)$.
\item\label{it:bounded-fit-2}
  Fit isotonic regression to $(r_1,y_1),\dots,(r_k,y_k),(r,C_*)$
  obtaining an isotonic calibrator~$f_*$.
\item\label{it:bounded-last}
  Set $\hat y_*:=f_*(r)$.
\end{enumerate}

The bounded IVAR is a simple modification of the inductive Venn--Abers predictor
as defined in \citet{Vovk/etal:2015-local}.
First, we relax the condition $y_i\in\{0,1\}$ of binary labels
by allowing the labels to take intermediate values, $y_i\in[0,1]$,
and then we scale the interval $[0,1]$ to arbitrary $[C_*,C^*]$.
As mentioned in the previous section,
non-inductive Venn--Abers regressors were introduced by \citet{Laan/Alaa:2024}.

\subsection{Unbounded regression}
\label{subsec:unbounded}

We again fix a regression algorithm
and are given a training set of size $l>k$.
Another parameter of our algorithm
is $m\in\{1,\dots,\lfloor(k-1)/2\rfloor\}$;
we are interested in a small $m$, such as $m=1$.
The corresponding \emph{inductive Venn--Abers regressor} (IVAR)
produces the following regression interval of the form \eqref{eq:interval}
for the test label:
\begin{enumerate}
\item
  Randomly split the training set of size $l>k$ into two parts,
  a \emph{proper training set} of size $l-k$
  and a \emph{calibration set} $z_1,\dots,z_k$ of size $k$,
  as before.
\item
  Train the regression algorithm on the proper training set
  obtaining a prediction rule $R:\mathbf{X}\to\R$.
\item
  Find the base predictions $r_i:=R(x_i)$, $i=1,\dots,k$, and $r:=R(x)$.
\item\label{it:f^*}
  Replace the $m$ smallest calibration labels $y_i$
  by the $(m+1)$th smallest calibration label $y_*$
  and replace the $m-1$ largest calibration labels $y_i$
  by the $m$th largest calibration label $y^*$.
  (Notice the asymmetry in the definitions of $y_*$ and $y^*$.)
  In other words, let the new calibration labels be
  \begin{equation}\label{eq:y'}
    y'_i
    :=
    \begin{cases}
      y_* & \text{if $y_i<y_*$}\\
      y^* & \text{if $y_i>y^*$}\\
      y_i & \text{otherwise.}
    \end{cases}
  \end{equation}
  (Notice that the recipe is still unambiguous when there are ties among $y_i$.)
\item\label{it:fit-1}
  Fit isotonic regression to $(r_1,y'_1),\dots,(r_k,y'_k),(r,y^*)$
  obtaining an isotonic calibrator~$f^*$.
\item
  Set $\hat y^*:=f^*(r)$.
\item\label{it:f_*}
  Replace the $m$ largest calibration labels $y_i$
  by the $(m+1)$th largest calibration label $y^*$
  and replace the $m-1$ smallest calibration labels $y_i$
  by the $m$th smallest calibration label $y_*$.
  (Notice that these $y_*$ and $y^*$ are different
  from those in step~\ref{it:f^*}.)
  In other words, let the new calibration labels be
  defined as \eqref{eq:y'} for the new $y_*$ and $y^*$.
\item\label{it:fit-2}
  Fit isotonic regression to $(r_1,y'_1),\dots,(r_k,y'_k),(r,y_*)$
  obtaining an isotonic calibrator~$f_*$.
\item
  Set $\hat y_*:=f_*(r)$.
\end{enumerate}

When fitting isotonic regression in steps~\ref{it:fit-1} and~\ref{it:fit-2}
(and the analogous steps, \ref{it:bounded-fit-1} and \ref{it:bounded-fit-2},
in the bounded IVAR),
we always use the standard pool-adjacent-violators algorithm
(PAVA; see, e.g., \citealt[Sect.~1.2]{Barlow/etal:1972}).
Notice that the new steps \ref{it:f^*} and \ref{it:f_*} in the IVAR
as compared with the bounded IVAR are really needed:
if we just set $y^*:=\infty$ and $y_*:=-\infty$,
the PAVA will produce $(-\infty,\infty)$ as the regression interval.

\section{Validity}

We often write $Z_1,\dots,Z_k$, where $Z_i=(X_i,Y_i)$,
for the calibration examples and $(X,Y)$ for the test example,
in order to emphasize that they are considered as random elements.

First we state a simpler property of validity,
the one for bounded regression
(including binary classification as a special case).
So we assume that the labels take values in $[C_*,C^*]$.
Let us say that a random variable $S$ is \emph{perfectly calibrated}
as a regression estimate of $Y$ if $S=\E(Y\mid S)$ a.s.
(In our terminology we follow, e.g., \citealt[Sect.~6.2.1]{Vovk/etal:2022book},
\citealt[Chap.~12]{Angelopoulos/etal:arXiv2411},
\citealt{Laan/Alaa:2024}, and \citealt{Laan/Alaa:2025};
another popular term is ``auto-calibration'': see, e.g.,
\citealt[Definition 3.1]{Kruger/Ziegel:2021}.)

\begin{remark}\upshape
  An equivalent definition of perfect calibration
  is that a random variable $S$ is said to be perfectly calibrated
  as a regression estimate of $Y$ if $S=\E(Y\mid\FFF)$ a.s.\
  for some $\sigma$-algebra $\FFF$
  (in which case we may also say that $S$ is \emph{ideal} relative to $\FFF$).
\end{remark}

A \emph{selector} for an IVAR is a random variable
that always belongs to the regression interval output by the IVAR.

\begin{theorem}\label{thm:simple-validity}
  For any bounded IVAR,
  there is a selector $S$ that is perfectly calibrated for the test label,
  i.e., $\E(Y\mid S)=S$ a.s.
\end{theorem}

Our interpretation of Theorem~\ref{thm:simple-validity}
is that the regression interval produced by the bounded IVAR
is approximately calibrated provided it is narrow enough.

Now we consider the case of unbounded regression
and continue to assume $2m<k$.
A \emph{selector} is defined as before,
and the \emph{Winsorized} test label is defined by
\begin{equation}\label{eq:Y'}
  Y'
  :=
  \begin{cases}
    Y_{(m)} & \text{if $Y<Y_{(m)}$}\\
    Y_{(k-m+1)} & \text{if $Y>Y_{(k-m+1)}$}\\
    Y & \text{otherwise,}
  \end{cases}
\end{equation}
where $Y_{(1)}\le\dots\le Y_{(k)}$ is the sequence $Y_1,\dots,Y_k$
of calibration labels sorted in the ascending order.
We can regard \eqref{eq:Y'} as a more feasible version of $Y$
corrected for the possibility of $Y$ being an outlier;
we make it easier to predict by restricting it
to be in the range of the calibration labels
(perhaps with a cushion).

\begin{remark}\upshape
  See \citet[Sect.~1]{Dixon:1960} and \citet[Sect.~14]{Tukey:1962}
  for the original definition of Winsorization,
  which we use in this paper.
  This definition is sometimes modified,
  as in, e.g., \citet[Sect.~2.2.2]{Wilcox:2012}.
  While the original definition is about moderating a dataset,
  the modified definition is about moderating the population.
\end{remark}

\begin{theorem}\label{thm:validity}
  We have $Y=Y'$ with probability at least $1-\frac{2m}{k+1}$.
  For any IVAR,
  there is a selector $S$ that is perfectly calibrated for the Winsorized test label $Y'$,
  i.e., $\E(Y'\mid S)=S$ a.s.
\end{theorem}

The obvious first statement of Theorem~\ref{thm:validity}
says that the Winsorized label $Y'$ is the same as the original label $Y$
with high probability assuming $m\ll k$.
The second statement asserts the validity of the IVAR
as a regression function for the Winsorized label.
See Appendix~\ref{app:proofs} for the proofs.
They will make clear how the symmetry in the definition \eqref{eq:Y'}
of the Winsorized test label
(involving the $m$th smallest and $m$th largest calibration label)
leads to the asymmetry in the definition of IVARs in Sect.~\ref{subsec:unbounded}.

An alternative parametrization of the IVAR
is in terms of $\epsilon:=2m/(k+1)$;
we then have $Y=Y'$ with probability at least $1-\epsilon$.
If a given $\epsilon\in[2/(k+1),1)$ is not of the form $2m/(k+1)$
for an integer $m$,
we decrease it as little as possible so that it takes this form.

\section{Algorithm}
\label{sec:algorithm}

The procedure described in Sect.~\ref{subsec:unbounded}
can be implemented very efficiently.
In our description of the algorithm,
we will just refer to \citet[Chap.~6, especially Sect.~6.5.3]{Vovk/etal:2022book},
which in turn follows \citet{Vovk/etal:2015-local}.
We will discuss computing $f^*$ and $f_*$,
which are denoted by $f^1$ and $f^0$, respectively,
in \citet[Sect.~6.5.3]{Vovk/etal:2022book}.
Remember that we moderate the original calibration labels $y_i$
to their less extreme versions $y'_i$ by \eqref{eq:y'};
after that, we forget the original labels and drop the primes.
But it should always be remembered that $y_i$ stand for the moderated labels.

First we see how to compute $F^1$, which determines $f^1=f^*$.
We are given the base predictions $r_i$ and the moderated labels $y_1,\dots,y_k$.
Define $k'$, $r'_1,\dots,r'_{k'}$, $w_1,\dots,w_{k'}$,
and $y'_1,\dots,y'_{k'}$ as in \citet[Sect.~6.5.3]{Vovk/etal:2022book}
(where base predictions were called scores and were denoted by $s_i$).
The CSD consisting of the points $P_i$, $i\in\{0,\dots,k'\}$, is defined as before, 
but now it is extended by adding the point $P_{-1}:=(-1,-y^*)$;
this corresponds to adding the test example to the calibration set
assuming that the base prediction for it is smaller
than the base prediction for any calibration example
while its label is $y^*$ (which is a most unusual combination).
Algorithm~6.3 in \citet[Sect.~6.5.3]{Vovk/etal:2022book}
will compute the corners of the resulting CSD $P_{-1},\dots,P_{k'}$
and Algorithm~6.4 will do the rest.

Computing $F^0$, which determines $f^0=f_*$, is analogous.
Now the CSD consisting of the points $P_i$, $i\in\{0,\dots,k'\}$,
is extended by adding the point $P_{k'+1}:=P_{k'}+(1,y_*)$;
this corresponds to adding the test example to the calibration set
assuming that its base prediction is larger than any calibration base prediction
while its label is $y_*$ (another most unusual combination).
Algorithm~6.5 in \citet[Sect.~6.5.3]{Vovk/etal:2022book}
will compute the corners of the resulting CSD $P_{0},\dots,P_{k'+1}$
and Algorithm~6.6 will do the rest.

Algorithm~6.7 in \citet[Sect.~6.5.3]{Vovk/etal:2022book}
shows how to arrange the calibration set
into a convenient binary search tree,
and Algorithm~6.8 shows how to use this tree
for computationally efficient prediction.
The computational complexity of these procedures
is summarized in the following theorem
(which ignores the complexity of training the base regressor and computing base predictions).

\begin{theorem}
  The computation time of our prediction algorithm is $O(k\log k)$ for preprocessing
  ($O(k)$ apart from sorting the calibration set).
  Once preprocessing is completed, processing each test object can be done in time $O(\log k)$.
\end{theorem}

\section{Merging a Regression Interval into a Single Value}
\label{sec:merging}

The IVAR introduced in Sect.~\ref{sec:IVAR}
achieves our goal of producing provably valid regression intervals
that have a potential to be predictively efficient
for efficient base algorithms.
However, in order to be able to compare the predictive efficiency
of our methods with traditional regression algorithms,
in this and following section we will define natural modifications
of the IVAR that produce point predictions.
In this section we see how to replace the regression intervals
output by IVARs by point predictions,
and in the following section we will see how we can combine
several IVARs to achieve both predictive and computational efficiency.

Given $y_*<\hat y_*<\hat y^*<y^*$,
let us see how to replace the regression interval $[\hat y_*,\hat y^*]$
with a single regression value $\hat y$.
Following the minimax approach of \citet[Sect.~6.4.3]{Vovk/etal:2022book},
we need to solve the equation
\begin{equation}\label{eq:minmax}
  \left(
    \hat y - y_*
  \right)^2
  -
  \left(
    \hat y_* - y_*
  \right)^2
  =
  \left(
    y^* - \hat y
  \right)^2
  -
  \left(
    y^* - \hat y^*
  \right)^2.
\end{equation}
It is clear that there is a unique solution $\hat y$ in the interval $[\hat y_*,\hat y^*]$,
since, over that interval,
the left-hand side of \eqref{eq:minmax} increases in $\hat y$ from 0 to a positive number,
while the right-hand side decreases from a positive number to 0.

After simplification \eqref{eq:minmax} becomes a linear equation in $\hat y$
with solution
\begin{equation}\label{eq:solution}
  \hat y
  =
  \frac
  {\hat y_*^2-\hat y^{*2}+2\hat y^*y^*-2\hat y_*y_*}
  {2(y^*-y_*)},
\end{equation}
and there is only one solution
(both in the interval $[\hat y_*,\hat y^*]$ and overall).

We can also solve \eqref{eq:minmax} approximately
obtaining a much more intuitive expression.
Rewriting \eqref{eq:minmax} as
\begin{equation*}
  2(\hat y-\hat y_*)(\hat y_*-y_*)+(\hat y-\hat y_*)^2
  =
  2(\hat y^*-\hat y)(y^*-\hat y^*)+(\hat y^*-\hat y)^2,
\end{equation*}
assuming $\hat y_*\approx\hat y^*$, and ignoring the quadratic terms,
we obtain the approximate equation
\[
  (\hat y-\hat y_*)(\hat y_*-y_*)
  =
  (\hat y^*-\hat y)(y^*-\hat y^*).
\]
Regrouping its terms as
\begin{equation*}
  \hat y(\hat y_*-y_*+y^*-\hat y^*)
  =
  \hat y_*(\hat y_*-y_*)
  +
  \hat y^*(y^*-\hat y^*),
\end{equation*}
we can write its solution as the weighted average
\begin{equation}\label{eq:approx}
  \hat y
  =
  \frac{\hat y_*-y_*}{\hat y_*-y_*+y^*-\hat y^*}
  \hat y_*
  +
  \frac{y^*-\hat y^*}{\hat y_*-y_*+y^*-\hat y^*}
  \hat y^*
\end{equation}
of $\hat y_*$ and $\hat y^*$.

\section{Cross Venn--Abers Regressors}
\label{sec:CVAR}

We define \emph{cross Venn--Abers regressors} (CVARs)
as in \citet[Sect.~6.4.4]{Vovk/etal:2022book}
except that the regression interval coming from each fold
is replaced by the regression estimate computed
as described in Sect.~\ref{sec:merging},
either as in \eqref{eq:solution} or as in \eqref{eq:approx}
(in our experiments in the next section we use the latter,
and we set $y_*:=y_{(m)}$ and $y^*:=y_{(k-m+1)}$,
according to \eqref{eq:Y'}).
Namely, we divide the training set into $K$ folds
of approximately equal size,
and use each fold in turn as the calibration set
while the union of the remaining folds is used as the proper training set.
The overall regression estimate is found as the arithmetic mean
of the $K$ regression estimates obtained by merging the regression intervals
coming from the $K$ IVARs corresponding to the $K$ folds.
In our experiments we set the parameter $K$ to 10.

\section{Experimental Results}
\label{sec:experiments}

As discussed at the end of Sect.~\ref{sec:literature},
in our experimental studies we concentrate on our point regressors,
namely the CVARs as defined in the previous section.
These point regressors no longer satisfy any formal properties of validity,
but the hope is that the validity properties of IVARs
will manifest themselves in superior performance of CVARs,
measured in the traditional way,
as compared with standard point regressors.

We evaluate our approach on a suite of controlled synthetic regression benchmarks
designed to isolate different statistical challenges.
We generate datasets with $n=10000$ examples
under the following generative scenarios.
Each dataset consists of $n$ examples
whose objects are vectors of $d=10$ features.
As before, the objects are denoted by $x_i$ and the corresponding labels by $y_i$,
$i=1,\dots,n$.
Across all datasets, the level of noise in the generated labels
is parametrized by $\sigma\in\{1,3\}$.
The results in this section represent the more challenging noise level $\sigma=3$,
and results for $\sigma=1$,
as well as for $n=1000$,
are included in Appendix~\ref{app:extra}.
Unless stated otherwise, all random draws are independent.

\paragraph{Bounded logistic dataset}
Each object is generated as $x_i \sim \NNN(0,I_{10})$,
the weight vector as $w \sim \NNN(0,I_{10})$,
and the Gaussian noise variables as $\xi_i \sim \NNN(0,\sigma^2)$.
(We parametrize the Gaussian distribution $\NNN$ by its mean and variance,
or its covariance matrix in the multidimensional case.)
The labels are then generated using the sigmoid function:
\[
  y_i
  :=
  \frac{10}{1 + \e^{-w^\top x_i}}
  +
  \xi_i.
\]
Notice that while the conditional expectation of $y_i$ is bounded
(belongs to $(0,10)$),
the labels are not bounded because of the Gaussian noise.
Since $y_i$ are ``almost bounded''
(the Gaussian distribution having thin tails),
this is the most benign case,
close to the setting of Sect.~\ref{subsec:bounded},
and our methods work best for it.

\paragraph{Linear Gaussian dataset}
We generate $x_i \sim \NNN(0,I_{10})$, $w \sim \NNN(0,I_{10})$,
and $\xi_i \sim \NNN(0,\sigma^2)$ as before.
The labels are then defined by
\[
  y_i := w^\top x_i + \xi_i.
\]

\paragraph{Nonlinear dataset}
Again $x_i \sim \NNN(0,I_{10})$ and $w \sim \NNN(0,I_{10})$.
The label of $x_i$ combines a linear contribution with smooth nonlinear components:
\[
  y_i := w^\top x_i
  + 2\sin(x_{i,1})
  + \tfrac{1}{2} x_{i,2}^2
  - \cos(2x_{i,3})
  + \xi_i,
\]
where $\xi_i \sim \NNN(0,\sigma^2)$ and $x_{i,j}$ denotes the $j$th feature of $x_i$.

\paragraph{Heteroscedastic noise dataset}
We again draw $x_i \sim \NNN(0,I_{10})$ and
$w \sim \NNN(0,I_{10})$ independently. However, the observation-noise
scale now depends on the object. Specifically,
\[
  y_i := w^\top x_i + \xi_i,
  \qquad
  \xi_i \sim \NNN\!\bigl(0,\sigma_i^2\bigr),
\]
where
\[
  \sigma_i
  :=
  \left(
    0.5+\left|x_{i,1}\right|
  \right)
  \sigma.
\]
Thus, $\sigma$ controls the overall noise level,
while the magnitude of the first feature determines the object-dependent noise scale.
This creates heteroscedastic, object-dependent uncertainty.

\paragraph{Heavy-tailed noise dataset}
Here $x_i \sim \NNN(0,I_{10})$ and $w \sim \NNN(0,I_{10})$ are
independently generated as before, but the additive noise follows a scaled
Student-$t$ distribution with three degrees of freedom:
\[
  y_i := w^\top x_i + \xi_i,
  \qquad
  \xi_i \sim \sigma t_3.
\]
Equivalently, $\xi_i := \sigma\varepsilon_i$,
where $\varepsilon_i \sim t_3$.
Thus, $\sigma$ controls the scale of the noise,
while the degrees of freedom remain fixed at three.
This produces infrequent but potentially very large deviations.

\paragraph{Outlier contamination dataset}
We generate $x_i \sim \NNN(0,I_{10})$ and $w \sim \NNN(0,I_{10})$,
and define
\[
  y_i := w^\top x_i + \xi_i,
\]
where $\xi_i$ follows a contamination model:
\[
  \xi_i
  \sim
  \begin{cases}
    \NNN(0,\sigma^2), & \text{with probability } 1-p,\\[4pt]
    \NNN(0,\tau^2), & \text{with probability } p,
  \end{cases}
\]
with $p:=0.01$ representing the outlier probability
and $\tau\gg1$ (we set $\tau:=10\sigma$) meaning
that the rare errors are extremely large.

\paragraph{Sparse signals dataset}
We again draw $x_i \sim \NNN(0,I_{10})$,
but the true vector of coefficients is sparse.
Let $s$ be the sparsity level (in our experiments we set $s:=1$);
we randomly select a support set $S\subset\{1,\dots,10\}$ with $\left|S\right| = s$
and define
\[
  w_j
  \sim
  \begin{cases}
    \NNN(0,1), & j \in S,\\
    0, & j \notin S.
  \end{cases}
\]
The labels follow
\[
  y_i := w^\top x_i + \xi_i,
  \qquad
  \xi_i \sim \NNN(0,\sigma^2).
\]

\paragraph{Covariate shift dataset}
This scenario uses a distribution-defined train--test split rather than
randomly partitioning a common IID sample. Of the $n$ observations, $80\%$
are independently generated as training objects from
\[
  x_i^{(\mathrm{train})}\sim\NNN(0,I_{10}),
\]
while the remaining $20\%$ are independently generated as test objects from
\[
  x_i^{(\mathrm{test})}\sim\NNN(\mathbf{1},I_{10}),
\]
where $\mathbf{1}$ denotes the $10$-dimensional vector of ones. A single
coefficient vector $w\sim\NNN(0,I_{10})$ generates the labels in both
splits:
\[
  y_i^{(s)} := w^\top x_i^{(s)}+\xi_i^{(s)},
  \qquad
  \xi_i^{(s)}\sim\NNN(0,\sigma^2),
  \qquad
  s\in\{\mathrm{train},\mathrm{test}\}.
\]
The features in both splits are subsequently standardized using the means
and standard deviations estimated from the training objects. Thus, the
conditional distribution $p(y\mid x)$ is unchanged across the two splits,
whereas the marginal object distribution $p(x)$ differs. Because the
training and test observations are not identically distributed, the IID
assumptions underlying our validity result do not hold in this scenario.
We therefore include it only as a robustness stress test under distribution
shift.

\paragraph{Friedman's datasets}
We also show results for Friedman's three synthetic datasets \citep{Friedman:1991}
as numbered by Breiman \citep[Sect.~3 and Appendix~B]{Breiman:1996}
and implemented in \texttt{scikit-learn}.
We use the \texttt{scikit-learn} implementation
with the default values of parameters except for the size
(namely, the default number of features is 10 for Friedman 1) and noise level $\sigma$.

\medskip

For each dataset we randomly split the data into 80\% training and 20\% testing,
standardize the features based on training statistics,
and repeat all experiments across multiple (namely, 100) random seeds to reduce variance.
(The covariate shift dataset is an exception in that splits are not random.)
At the end of the section we also briefly discuss results
for the other combinations of $n\in\{1000,10000\}$ and $\sigma\in\{1,3\}$.

We compare our methods against seven widely used regression baselines
implemented in \texttt{scikit-learn},
including linear regressors (Linear Regression, Ridge, Lasso, Elastic Net),
kernel-based methods (Support Vector Regression, SVR, with the RBF kernel),
and tree-based ensembles (Random Forest and Gradient Boosting).
Linear Regression in fact implements Least Squares,
and Ridge (or Ridge Regression), Lasso, and Elastic Net
add various elements of regularization.
All models are trained with default or widely adopted hyperparameters
to reflect typical practitioner usage.
Performance is measured using root mean squared error (RMSE) on the held-out test set.
We report mean RMSE across repeated trials for each scenario to assess their accuracy.

The overall procedure can be summarized as follows.
For each dataset and random seed,
we first split the data into a training portion (80\%) and a held-out test portion (20\%).
All model fitting, calibration, and cross-fitting are performed only on the training portion.
Final metrics (such as RMSE) are evaluated only on the held-out test portion.
The uncalibrated baseline labelled ``none'' in our tables
is trained once on the full training portion.
The CVAR methods use the same training portion but internally perform $K = 10$ fold cross-fitting:
each fold is used as a calibration set while the remaining folds are used as proper training data.
The fold-specific IVAR intervals are merged into point predictions
as described in Sect.~\ref{sec:merging} and averaged across folds.

Tables~\ref{tab:BL10K3}--\ref{tab:F3_10K3} show the average RMSE
over 100 trials for each synthetic dataset
with $\sigma=3$ and $n=10000$ examples
and for each base algorithm without calibration (\emph{none})
or calibrated using our CVAR method with 10 folds and parameters $m\in\{1,10\}$
(denoted \emph{CVAR1} and \emph{CVAR10}, respectively, in the tables;
using $m=100$ never leads to improvements as compared with the smaller $m$
in our experiments).
Apart from the average RMSE we also report the standard error of the mean (SEM)
obtained by dividing the sample standard deviation of the RMSE over the 100 trials
by $\sqrt{100}=10$; therefore, each cell in our tables has the form
\[
  \text{average RMSE} \pm \text{SEM}.
\]
These intervals reassure us that most of our comparisons
are not unduly affected by randomness,
but in our summaries we only use the averages.
The smallest average in each row of our tables is shown in boldface.
For each table we also report the average of each column,
which is the average of the average RMSE over the seven regression algorithms.

\begin{remark}\upshape
  The results for Linear Regression and Ridge coincide
  in Tables~\ref{tab:BL10K3}--\ref{tab:F3_10K3},
  but they are sometimes slightly different in Appendix~\ref{app:extra}.
  The reason for this closeness is that our features are constructed as independent,
  so multicollinearity is unlikely,
  and the sample size is large relative to the feature count;
  this makes Least Squares stable,
  and so the default ridge penalty has relatively little impact.
\end{remark}

The last rows in Tables~\ref{tab:BL10K3}--\ref{tab:F3_10K3}
show that both of our methods, CVAR1 and CVAR10, attain better RMSE scores on average
when compared with uncalibrated algorithms,
especially for the bounded logistic, linear Gaussian, nonlinear, covariate shift,
Friedman 1, and Friedman 2 datasets
(the results for the last dataset, however, are very irregular,
and in two cases our methods significantly lower the quality of predictions).
In many rows the base algorithm performs better,
but in a typical row of a typical table either our algorithms perform better
or they lose little.
For the ``almost bounded'' bounded logistic dataset our methods
improve the base predictions in all rows.

\begin{table}
  \caption{Bounded logistic dataset}\label{tab:BL10K3}
  \medskip
  \begin{center}
  \begin{tabular}{lccc}
    \toprule
    & none & CVAR1 & CVAR10 \\
    \midrule
    Elastic Net & $3.809 \pm 0.015$ & $\textbf{3.204} \pm 0.007$ & $3.205 \pm 0.007$ \\
    Gradient Boosting & $3.239 \pm 0.010$ & $\textbf{3.123} \pm 0.006$ & $3.125 \pm 0.006$ \\
    Lasso & $4.001 \pm 0.018$ & $\textbf{3.587} \pm 0.017$ & $3.588 \pm 0.017$ \\
    Linear Regression & $3.257 \pm 0.012$ & $\textbf{3.010} \pm 0.005$ & $3.011 \pm 0.005$ \\
    Random Forest & $3.231 \pm 0.008$ & $\textbf{3.182} \pm 0.007$ & $3.184 \pm 0.007$ \\
    Ridge & $3.257 \pm 0.012$ & $\textbf{3.010} \pm 0.005$ & $3.011 \pm 0.005$ \\
    SVR (RBF) & $3.143 \pm 0.007$ & $\textbf{3.065} \pm 0.005$ & $3.066 \pm 0.005$ \\
    \midrule
    average & $3.420$ & $\textbf{3.169}$ & $3.170$ \\
    \bottomrule
  \end{tabular}
  \end{center}
\end{table}

\begin{table}
  \caption{Linear Gaussian dataset}\label{tab:LG10K3}
  \medskip
  \begin{center}
  \begin{tabular}{lccc}
    \toprule
    & none & CVAR1 & CVAR10 \\
    \midrule
    Elastic Net & $3.510 \pm 0.015$ & $\textbf{3.147} \pm 0.006$ & $3.155 \pm 0.006$ \\
    Gradient Boosting & $3.073 \pm 0.006$ & $\textbf{3.071} \pm 0.006$ & $3.080 \pm 0.006$ \\
    Lasso & $3.707 \pm 0.018$ & $\textbf{3.418} \pm 0.014$ & $3.424 \pm 0.014$ \\
    Linear Regression & $\textbf{3.002} \pm 0.005$ & $3.019 \pm 0.005$ & $3.029 \pm 0.005$ \\
    Random Forest & $3.138 \pm 0.006$ & $\textbf{3.127} \pm 0.007$ & $3.136 \pm 0.007$ \\
    Ridge & $\textbf{3.002} \pm 0.005$ & $3.019 \pm 0.005$ & $3.029 \pm 0.005$ \\
    SVR (RBF) & $\textbf{3.087} \pm 0.005$ & $3.096 \pm 0.006$ & $3.104 \pm 0.006$ \\
    \midrule
    average & $3.217$ & $\textbf{3.128}$ & $3.137$ \\
    \bottomrule
  \end{tabular}
  \end{center}
\end{table}

\begin{table}
  \caption{Nonlinear dataset}\label{tab:NL10K3}
  \medskip
  \begin{center}
  \begin{tabular}{lccc}
    \toprule
    & none & CVAR1 & CVAR10 \\
    \midrule
    Elastic Net & $3.718 \pm 0.015$ & $\textbf{3.327} \pm 0.006$ & $3.335 \pm 0.006$ \\
    Gradient Boosting & $3.100 \pm 0.006$ & $\textbf{3.098} \pm 0.006$ & $3.110 \pm 0.006$ \\
    Lasso & $3.904 \pm 0.018$ & $\textbf{3.590} \pm 0.013$ & $3.595 \pm 0.014$ \\
    Linear Regression & $\textbf{3.199} \pm 0.005$ & $3.211 \pm 0.006$ & $3.221 \pm 0.006$ \\
    Random Forest & $3.194 \pm 0.008$ & $\textbf{3.185} \pm 0.008$ & $3.196 \pm 0.008$ \\
    Ridge & $\textbf{3.199} \pm 0.005$ & $3.211 \pm 0.006$ & $3.221 \pm 0.006$ \\
    SVR (RBF) & $\textbf{3.139} \pm 0.006$ & $3.153 \pm 0.006$ & $3.162 \pm 0.006$ \\
    \midrule
    average & $3.350$ & $\textbf{3.254}$ & $3.263$ \\
    \bottomrule
  \end{tabular}
  \end{center}
\end{table}

\begin{table}
  \caption{Heteroscedastic noise dataset}\label{tab:HSN10K3}
  \medskip
  \begin{center}
  \begin{tabular}{lccc}
    \toprule
    & none & CVAR1 & CVAR10 \\
    \midrule
    Elastic Net & $4.654 \pm 0.016$ & $\textbf{4.386} \pm 0.011$ & $4.392 \pm 0.011$ \\
    Gradient Boosting & $4.371 \pm 0.012$ & $\textbf{4.351} \pm 0.012$ & $4.359 \pm 0.012$ \\
    Lasso & $4.804 \pm 0.017$ & $\textbf{4.586} \pm 0.015$ & $4.589 \pm 0.015$ \\
    Linear Regression & $\textbf{4.289} \pm 0.011$ & $4.295 \pm 0.011$ & $4.302 \pm 0.011$ \\
    Random Forest & $4.430 \pm 0.012$ & $\textbf{4.400} \pm 0.012$ & $4.407 \pm 0.012$ \\
    Ridge & $\textbf{4.289} \pm 0.011$ & $4.295 \pm 0.011$ & $4.302 \pm 0.011$ \\
    SVR (RBF) & $\textbf{4.358} \pm 0.011$ & $4.361 \pm 0.011$ & $4.366 \pm 0.012$ \\
    \midrule
    average & $4.456$ & $\textbf{4.382}$ & $4.388$ \\
    \bottomrule
  \end{tabular}
  \end{center}
\end{table}

\begin{table}
  \caption{Heavy-tailed noise dataset}\label{tab:HTN10K3}
  \medskip
  \begin{center}
  \begin{tabular}{lccc}
    \toprule
    & none & CVAR1 & CVAR10 \\
    \midrule
    Elastic Net & $5.444 \pm 0.076$ & $5.216 \pm 0.077$ & $\textbf{5.210} \pm 0.077$ \\
    Gradient Boosting & $5.207 \pm 0.077$ & $\textbf{5.185} \pm 0.077$ & $5.187 \pm 0.077$ \\
    Lasso & $5.578 \pm 0.075$ & $5.399 \pm 0.075$ & $\textbf{5.388} \pm 0.075$ \\
    Linear Regression & $\textbf{5.120} \pm 0.078$ & $5.135 \pm 0.078$ & $5.132 \pm 0.078$ \\
    Random Forest & $5.280 \pm 0.077$ & $\textbf{5.237} \pm 0.077$ & $5.238 \pm 0.077$ \\
    Ridge & $\textbf{5.120} \pm 0.078$ & $5.135 \pm 0.078$ & $5.132 \pm 0.078$ \\
    SVR (RBF) & $\textbf{5.173} \pm 0.077$ & $5.189 \pm 0.077$ & $5.181 \pm 0.077$ \\
    \midrule
    average & $5.275$ & $5.214$ & $\textbf{5.210}$ \\
    \bottomrule
  \end{tabular}
  \end{center}
\end{table}

\begin{table}
  \caption{Outlier contamination dataset}\label{tab:OC10K3}
  \medskip
  \begin{center}
  \begin{tabular}{lccc}
    \toprule
    & none & CVAR1 & CVAR10 \\
    \midrule
    Elastic Net & $4.674 \pm 0.040$ & $\textbf{4.397} \pm 0.042$ & $4.399 \pm 0.042$ \\
    Gradient Boosting & $4.389 \pm 0.042$ & $\textbf{4.361} \pm 0.042$ & $4.370 \pm 0.042$ \\
    Lasso & $4.826 \pm 0.039$ & $4.610 \pm 0.040$ & $\textbf{4.604} \pm 0.040$ \\
    Linear Regression & $\textbf{4.290} \pm 0.043$ & $4.307 \pm 0.043$ & $4.311 \pm 0.043$ \\
    Random Forest & $4.456 \pm 0.042$ & $\textbf{4.412} \pm 0.042$ & $4.419 \pm 0.042$ \\
    Ridge & $\textbf{4.290} \pm 0.043$ & $4.307 \pm 0.043$ & $4.311 \pm 0.043$ \\
    SVR (RBF) & $\textbf{4.353} \pm 0.042$ & $4.369 \pm 0.042$ & $4.366 \pm 0.042$ \\
    \midrule
    average & $4.468$ & $\textbf{4.395}$ & $4.397$ \\
    \bottomrule
  \end{tabular}
  \end{center}
\end{table}

\begin{table}
  \caption{Sparse signals dataset}\label{tab:SHDS10K3}
  \medskip
  \begin{center}
  \begin{tabular}{lccc}
    \toprule
    & none & CVAR1 & CVAR10 \\
    \midrule
    Elastic Net & $3.049 \pm 0.007$ & $\textbf{2.998} \pm 0.005$ & $2.998 \pm 0.005$ \\
    Gradient Boosting & $3.009 \pm 0.005$ & $\textbf{3.004} \pm 0.005$ & $3.005 \pm 0.005$ \\
    Lasso & $3.079 \pm 0.009$ & $\textbf{3.024} \pm 0.006$ & $3.025 \pm 0.006$ \\
    Linear Regression & $\textbf{2.995} \pm 0.005$ & $2.997 \pm 0.005$ & $2.997 \pm 0.005$ \\
    Random Forest & $3.046 \pm 0.005$ & $\textbf{3.018} \pm 0.005$ & $3.019 \pm 0.005$ \\
    Ridge & $\textbf{2.995} \pm 0.005$ & $2.997 \pm 0.005$ & $2.997 \pm 0.005$ \\
    SVR (RBF) & $3.038 \pm 0.005$ & $3.022 \pm 0.005$ & $\textbf{3.022} \pm 0.005$ \\
    \midrule
    average & $3.030$ & $\textbf{3.009}$ & $3.009$ \\
    \bottomrule
  \end{tabular}
  \end{center}
\end{table}

\begin{table}
  \caption{Covariate shift dataset}\label{tab:CS10K3}
  \medskip
  \begin{center}
  \begin{tabular}{lccc}
    \toprule
    & none & CVAR1 & CVAR10 \\
    \midrule
    Elastic Net & $3.887 \pm 0.050$ & $\textbf{3.356} \pm 0.023$ & $3.380 \pm 0.025$ \\
    Gradient Boosting & $\textbf{3.213} \pm 0.011$ & $3.251 \pm 0.020$ & $3.281 \pm 0.023$ \\
    Lasso & $4.191 \pm 0.065$ & $\textbf{3.814} \pm 0.044$ & $3.827 \pm 0.044$ \\
    Linear Regression & $\textbf{3.007} \pm 0.004$ & $3.112 \pm 0.018$ & $3.143 \pm 0.021$ \\
    Random Forest & $\textbf{3.371} \pm 0.019$ & $3.395 \pm 0.024$ & $3.423 \pm 0.027$ \\
    Ridge & $\textbf{3.007} \pm 0.004$ & $3.112 \pm 0.018$ & $3.143 \pm 0.021$ \\
    SVR (RBF) & $\textbf{3.600} \pm 0.043$ & $3.652 \pm 0.046$ & $3.669 \pm 0.048$ \\
    \midrule
    average & $3.468$ & $\textbf{3.384}$ & $3.410$ \\
    \bottomrule
  \end{tabular}
  \end{center}
\end{table}

\begin{table}
  \caption{Friedman 1 dataset}\label{tab:F1_10K3}
  \medskip
  \begin{center}
  \begin{tabular}{lccc}
    \toprule
    & none & CVAR1 & CVAR10 \\
    \midrule
    Elastic Net & $4.383 \pm 0.008$ & $\textbf{3.871} \pm 0.007$ & $3.879 \pm 0.007$ \\
    Gradient Boosting & $\textbf{3.139} \pm 0.006$ & $3.158 \pm 0.006$ & $3.176 \pm 0.006$ \\
    Lasso & $4.353 \pm 0.008$ & $\textbf{3.959} \pm 0.007$ & $3.967 \pm 0.007$ \\
    Linear Regression & $3.864 \pm 0.007$ & $\textbf{3.857} \pm 0.007$ & $3.865 \pm 0.007$ \\
    Random Forest & $\textbf{3.276} \pm 0.006$ & $3.286 \pm 0.006$ & $3.303 \pm 0.006$ \\
    Ridge & $3.864 \pm 0.007$ & $\textbf{3.857} \pm 0.007$ & $3.865 \pm 0.007$ \\
    SVR (RBF) & $\textbf{3.226} \pm 0.006$ & $3.254 \pm 0.006$ & $3.269 \pm 0.006$ \\
    \midrule
    average & $3.729$ & $\textbf{3.606}$ & $3.618$ \\
    \bottomrule
  \end{tabular}
  \end{center}
\end{table}

\begin{table}
  \caption{Friedman 2 dataset}\label{tab:F2_10K3}
  \medskip
  \begin{center}
  \begin{tabular}{lccc}
    \toprule
    & none & CVAR1 & CVAR10 \\
    \midrule
    Elastic Net & $181.783 \pm 0.340$ & $\textbf{82.677} \pm 0.173$ & $84.128 \pm 0.173$ \\
    Gradient Boosting & $\textbf{15.432} \pm 0.057$ & $48.262 \pm 0.120$ & $50.548 \pm 0.152$ \\
    Lasso & $137.838 \pm 0.233$ & $\textbf{82.665} \pm 0.171$ & $84.119 \pm 0.171$ \\
    Linear Regression & $137.820 \pm 0.233$ & $\textbf{82.709} \pm 0.171$ & $84.160 \pm 0.171$ \\
    Random Forest & $\textbf{6.550} \pm 0.019$ & $47.284 \pm 0.120$ & $49.619 \pm 0.152$ \\
    Ridge & $137.820 \pm 0.233$ & $\textbf{82.709} \pm 0.171$ & $84.160 \pm 0.171$ \\
    SVR (RBF) & $140.878 \pm 0.536$ & $\textbf{105.209} \pm 0.381$ & $105.363 \pm 0.403$ \\
    \midrule
    average & $108.303$ & $\textbf{75.931}$ & $77.443$ \\
    \bottomrule
  \end{tabular}
  \end{center}
\end{table}

\begin{table}
  \caption{Friedman 3 dataset}\label{tab:F3_10K3}
  \medskip
  \begin{center}
  \begin{tabular}{lccc}
    \toprule
    & none & CVAR1 & CVAR10 \\
    \midrule
    Elastic Net & $3.022 \pm 0.005$ & $3.022 \pm 0.005$ & $\textbf{3.022} \pm 0.005$ \\
    Gradient Boosting & $3.023 \pm 0.005$ & $3.016 \pm 0.005$ & $\textbf{3.015} \pm 0.005$ \\
    Lasso & $3.022 \pm 0.005$ & $3.022 \pm 0.005$ & $\textbf{3.022} \pm 0.005$ \\
    Linear Regression & $3.013 \pm 0.005$ & $3.015 \pm 0.005$ & $\textbf{3.013} \pm 0.005$ \\
    Random Forest & $3.111 \pm 0.005$ & $3.023 \pm 0.005$ & $\textbf{3.022} \pm 0.005$ \\
    Ridge & $3.013 \pm 0.005$ & $3.015 \pm 0.005$ & $\textbf{3.013} \pm 0.005$ \\
    SVR (RBF) & $3.023 \pm 0.005$ & $3.017 \pm 0.005$ & $\textbf{3.016} \pm 0.005$ \\
    \midrule
    average & $3.032$ & $3.018$ & $\textbf{3.018}$ \\
    \bottomrule
  \end{tabular}
  \end{center}
\end{table}

We report the results for $\sigma\in\{1,3\}$ and $n=1000$
and for $\sigma=1$ and $n=10000$ settings in Appendix~\ref{app:extra}.
The results for $n=10000$ and $\sigma=1$ are similar to the results in this section,
while the results for $n=1000$ are mixed for $\sigma=3$
and weaker for our methods in the case of low noise, $\sigma=1$.
It appears that our methods work well for larger datasets.

In Appendix~\ref{app:extra} we also give results
for four real-life datasets.
Our methods tend to improve the performance of the base algorithms
unless the dataset is small, but the tendency is weak.

\subsection*{The widths of regression intervals}

In the main part of this section we only studied
the performance of our point predictors, CVAR1 and CVAR10,
ultimately based on inductive Venn--Abers prediction
but not having any formal properties of validity.
The reason for this choice was that the properties of validity that we established for IVARs
are very different from those used in existing literature,
and any comparisons with existing algorithms would be of very limited interest.
For example, the main algorithm (Algorithm~2) in \citet{Laan/Alaa:2024}
outputs prediction rather than regression intervals
and, therefore, produces intervals that are wider than ours by orders of magnitude
for a typical dataset considered earlier in this section.
It does not mean at all that our algorithms are better
since, unlike Algorithm~2 in \citet{Laan/Alaa:2024},
they do not guarantee a good coverage probability.

\begin{table}
  \caption{Regression interval widths for the naive IVAR, IVAR1, and IVAR10
    as described in the text}\label{tab:widths}
  \medskip
  \begin{center}
  \begin{tabular}{lccc}
    \toprule
    Dataset & naive IVAR & IVAR1 & IVAR10 \\
    \midrule
    Bounded logistic & 0.163 & 0.164 & 0.127 \\
    Linear Gaussian & 0.206 & 0.206 & 0.152 \\
    Nonlinear & 0.216 & 0.217 & 0.159 \\
    Heteroscedastic noise & 0.221 & 0.222 & 0.140 \\
    Heavy-tailed noise & 0.295 & 0.302 & 0.132 \\
    Outlier contamination & 0.381 & 0.392 & 0.126 \\
    Sparse signals & 0.105 & 0.106 & 0.077 \\
    Covariate shift & 0.275 & 0.276 & 0.193 \\
    Friedman 1 & 0.231 & 0.231 & 0.178 \\
    Friedman 2 & 0.255 & 0.255 & 0.232 \\
    Friedman 3 & 0.075 & 0.076 & 0.055 \\
    \bottomrule
    \end{tabular}
  \end{center}
\end{table}

Table~\ref{tab:widths} compares the widths
(normalized by the standard deviation of the training labels)
of regression intervals produced by IVAR1 and IVAR10,
i.e., the IVAR with $m=1$ and $m=10$, respectively.
Even though such a comparison is internal to this paper,
it is somewhat instructive as it shows that IVAR10 produces regression intervals
that are considerably shorter.
The column ``naive IVAR'' in the table gives results for the modification of the IVAR
in which $y_*$ and $y^*$ are replaced by the smallest and largest labels,
respectively, in the calibration set (as suggested by a reviewer).
It can be checked that the naive IVAR never outputs regression intervals
that are wider than the regression intervals output by the IVAR1 on the same data,
but the former can be narrower
(this can be deduced from Lemma~\ref{lem:monotonicity}
given in Appendix~\ref{app:proofs}).
Of course, the price to pay for the potentially narrower regression intervals
is that the naive IVAR does not enjoy
the same provable properties of validity as the IVAR1.
In our experiments summarized in Table~\ref{tab:widths},
the IVAR10 produces by far the narrowest regression intervals,
while the naive IVAR and IVAR1 have almost identical widths.
The IVAR1 produces regression intervals that are noticeably wider than the naive IVAR's
only for the heavy-tailed noise and outlier contamination datasets,
but even for them the difference between the IVAR1 and IVAR10
dwarfs the difference between the naive IVAR and IVAR1.
The results of Table~\ref{tab:widths} are for the Gradient Boosting base regressor,
$n:=10000$, and $\sigma:=3$,
but these results are typical.
The widths of the regression intervals are computed from the training set only;
the training set is split randomly into $K=10$ folds
and each fold in turn is used as the calibration set (similarly to CVARs);
the widths of the resulting regression intervals are then averaged.

\section{Conclusion}

This paper presents new validity guarantees for regression estimation.
Computational experiments show that this leads to a limited improvement
in the performance of standard point regressors on large datasets.

These are some directions for further research.
\begin{itemize}
\item
  An alternative merging procedure to the ones used in Sect.~\ref{sec:CVAR}
  is to try and merge the regression intervals coming from the $K$ folds directly
  (in a minimax manner, as in \citealt[Sect.~6.4.5]{Vovk/etal:2022book})
  in order to obtain an overall regression estimate,
  without the intermediate step of merging each regression interval.
  It would be interesting to compare it experimentally
  with the procedure used in this paper.
\item
  In this paper we explore the predictive efficiency of the CVAR in simulation and empirical studies.
  Alternatively, we could try and prove theoretical results along the lines of the Burnaev--Wasserman programme,
  as described in \citet[Sects.~2.5 and~2.9.7]{Vovk/etal:2022book}.
  As a first step, we may assume that the base predictions $r$
  coincide with the true regression function, $r:=\E(Y\mid X=x)$.
\item
  Our methods tend to improve significantly the quality of predictions
  made by Lasso and Elastic Net.
  A recurring pattern of the boldface (i.e., best) entries in our tables
  is ``CVAR1, none, CVAR1, none, none, none, none'' (from top to bottom).
  In the main paper this pattern occurs in Table~\ref{tab:CS10K3},
  and it occurs in 18 out of 37 tables
  in Appendix~\ref{app:extra}.
  For this pattern our methods help only for Lasso and Elastic Net.
  However, it can be argued that applying our methods destroys,
  at least partially,
  a valued property of Lasso and Elastic Net,
  automatic feature selection.
  Optimizing the number of features is another interesting desideratum,
  alongside predictive efficiency.
\item
  Results for bagged models are more mixed.
  CVAR1 improves on Random Forest 8 times out of 11
  in Tables~\ref{tab:BL10K3}--\ref{tab:F3_10K3},
  while CVAR10 improves on it 7 times out of 11;
  however, the improvements are generally modest and there remain
  several datasets where calibration worsens Random Forest predictions.
  One of the reviewers suggested that there is an element of bagging
  in cross Venn--Abers regression as we average point predictions coming from different folds.
  It would be interesting to disentangle the effects of bagging and calibration
  in CVARs' performance.
\item
  In our empirical studies we focused on point predictors, CVARs,
  derived from IVARs but not enjoying any formal properties of validity.
  Only in Table~\ref{tab:widths} did we compare the widths of regression intervals
  coming from IVAR1 and IVAR10.
  Perhaps the most important direction of further research
  is the empirical study of IVARs, both their validity and efficiency.
  While validity is asserted in Theorem~\ref{thm:validity},
  checking it empirically might shed further light on its meaning.
  Exploring the efficiency of IVARs requires designing suitable ways
  of measuring it,
  along the lines of proper scoring rules (see, e.g., \citealt{Dawid:ESS2006PF})
  in probability forecasting
  and conditionally proper criteria of efficiency \citep[Sect.~3.1.5]{Vovk/etal:2022book}
  in conformal prediction.
  A natural question is how to balance encouraging short regression intervals
  and a good performance of, say, their midpoints (e.g., in the sense of RMSE).
\end{itemize}

\subsection*{Acknowledgements}

We are grateful to anonymous referees for their criticism and advice.
We used Claude Opus 5 for debugging our code and checking final versions of the paper
for errors and inconsistencies.

\appendix
\section{Proofs}
\label{app:proofs}

\subsection{Proof of Theorem~\ref{thm:simple-validity}}

The selector $S$ that we use for demonstrating the theorem
is a modification of steps~\ref{it:bounded-fit-1}--\ref{it:bounded-last}
in the description of the bounded IVAR
corresponding to using the true test label.
Namely, we fit isotonic regression to $(r_1,y_1),\dots,(r_k,y_k),(r,y)$,
where $y$ is the true label of the test object,
obtaining an isotonic calibrator $f$.
Set $S:=f(r)$.

First we need to check that $S$ is indeed a selector,
namely that $f_*(r)\le S\le f^*(r)$.
This follows from $C_*\le y\le C^*$
and the monotonicity property of PAVA
given by the following lemma.

\begin{lemma}\label{lem:monotonicity}
  If $y_i\le y'_i$ for all $i\in\{1,\dots,n\}$,
  then $f\le f'$,
  where $f$ (resp.\ $f'$) is the isotonic regression
  fitted to $(r_1,y_1),\dots,(r_n,y_n)$
  (resp.\ to $(r_1,y'_1),\dots,(r_n,y'_n)$).
\end{lemma}

\begin{proof}
  This is an immediate corollary of the max-min formulas
  (as given in, e.g., \citealt[p.~19]{Barlow/etal:1972}).
\end{proof}

Now let us check that $\E(Y\mid S)=S$ even conditionally
on the bag $B=\lbag Z_1,\dots,Z_k,Z\rbag$ of $k$ calibration examples and a test example
under the uniform probability measure on all $(k+1)!$ orderings of the bag
(see \citealt[Lemma~A.3]{Vovk/etal:2022book}).
The value of $S$ determines the solution block containing the test example.
Therefore, both $S$ and $\E(Y\mid S,B)$ will equal the arithmetic mean
of the labels in that solution block.

\subsection{Proof of Theorem~\ref{thm:validity}}

Fix an IVAR.
The selector $S$ is produced by the following ideal picture.
Let $y$ be the true label of the test object $x$.
Then the selector $S$ associated with the IVAR
is defined by the following recipe
applied after splitting the training set
and training the base regression algorithm on the proper training set
(which gives us a prediction rule).
\begin{enumerate}
\item\label{it:validity}
  Winsorize the labels, namely:
  \begin{itemize}
  \item
    replace the $m$ largest labels
    in the \emph{augmented calibration sequence}
    \[(x_1,y_1),\dots,(x_k,y_k),(x,y)\]
    by the $(m+1)$th largest label;
  \item
    replace the $m$ smallest labels
    in the augmented calibration sequence
    by the $(m+1)$th smallest label.
  \end{itemize}
\item
  Find the base predictions $r_1,\dots,r_k,r$ of the objects $x_1,\dots,x_k,x$
  using the prediction rule.
\item
  Fit isotonic regression $f$ to $(r_1,y_1),\dots,(r_k,y_k),(r,y)$
  (with the Winsorized labels).
\item
  Set $S:=f(r)$.
\end{enumerate}
Step \ref{it:validity} regularizes the labels moderating the $2m$ most extreme ones.
The recipe treats all elements of the augmented calibration sequence symmetrically,
which is essential in the proof of $\E(Y'\mid S)=S$;
let us call it the \emph{symmetric recipe}.

First let us check that $S$ is indeed a selector for the IVAR.
Consider three cases:
\begin{itemize}
\item
  If $y$ is one of the $m$ largest labels in the augmented calibration sequence,
  the definition of IVAR and the symmetric recipe modify the labels in the same way;
  in particular, the $m-1$ largest calibration labels and the test label
  are replaced by $y^*$,
  and the $m$ smallest calibration labels are replaced by $y_*$.
  In this case $S=\hat y^*$.
\item
  Analogously, if $y$ is one of the $m$ smallest labels in the augmented calibration sequence,
  we have $S=\hat y_*$.
\item
  Otherwise,
  the monotonicity property of isotonic regression
  given in Lemma~\ref{lem:monotonicity} above
  implies that we still have $S\in[\hat y_*,\hat y^*]$.
\end{itemize}
In all three cases, $S\in[\hat y_*,\hat y^*]$,
which means that $S$ is a selector.

\begin{remark}\upshape
  The first two cases considered in the proof demonstrate
  the tightness, in some sense, of the regression interval $[\hat y_*,\hat y^*]$.
\end{remark}

It remains to prove $\E(Y'\mid S)=S$.
We prove this equality conditionally on the set of all permutations
(equiprobable)
of a given augmented calibration sequence,
where a permutation $\pi$ of $\{1,\dots,k+1\}$ makes $z_{\pi^{-1}(k+1)}$ the test example.
The Winsorized test label \eqref{eq:Y'} over the permutations
is the same random variable as the test label produced
according to the symmetric recipe applied to the permutations.
It remains to remember that the isotonic regression
at the base prediction $r$ for the test object
is the mean of the labels in the augmented calibration sequence
in the same solution block as the test example
\citep[pp.~13--15]{Barlow/etal:1972}.
Solution blocks are level sets of the isotonic regression,
which implies $S$ being the conditional average of the Winsorized test label
given $S$.

\section{Further Experimental Results}
\label{app:extra}

In Tables~\ref{tab:BL10K1}--\ref{tab:F3_10K1}
we show the experimental results in the case of $\sigma=1$
parallel to those reported in Sect.~\ref{sec:experiments} of the main paper;
in particular, we still have $n=10000$.
Our comments in the main paper are also applicable here.
In particular, it is still true that on average
both CVAR1 and CVAR10 produce better results than the base algorithms
in all 11 tables.
The level of noise does not appear to be a crucial parameter.

\begin{table}
  \caption{Bounded logistic dataset ($\sigma=1$)}\label{tab:BL10K1}
  \medskip
  \begin{center}
  \begin{tabular}{lccc}
    \toprule
    & none & CVAR1 & CVAR10 \\
    \midrule
    Elastic Net & $2.551 \pm 0.021$ & $\textbf{1.505} \pm 0.012$ & $1.508 \pm 0.012$ \\
    Gradient Boosting & $1.532 \pm 0.018$ & $\textbf{1.293} \pm 0.008$ & $1.297 \pm 0.008$ \\
    Lasso & $2.830 \pm 0.024$ & $\textbf{2.201} \pm 0.025$ & $2.203 \pm 0.025$ \\
    Linear Regression & $1.596 \pm 0.021$ & $\textbf{1.030} \pm 0.002$ & $1.035 \pm 0.002$ \\
    Random Forest & $1.497 \pm 0.017$ & $\textbf{1.439} \pm 0.013$ & $1.443 \pm 0.013$ \\
    Ridge & $1.596 \pm 0.021$ & $\textbf{1.030} \pm 0.002$ & $1.035 \pm 0.002$ \\
    SVR (RBF) & $1.211 \pm 0.009$ & $\textbf{1.089} \pm 0.002$ & $1.093 \pm 0.002$ \\
    \midrule
    average & $1.830$ & $\textbf{1.370}$ & $1.373$ \\
    \bottomrule
  \end{tabular}
  \end{center}
\end{table}

\begin{table}
  \caption{Linear Gaussian dataset ($\sigma=1$)}\label{tab:LG10K1}
  \medskip
  \begin{center}
  \begin{tabular}{lccc}
    \toprule
    & none & CVAR1 & CVAR10 \\
    \midrule
    Elastic Net & $2.079 \pm 0.024$ & $\textbf{1.395} \pm 0.009$ & $1.414 \pm 0.010$ \\
    Gradient Boosting & $\textbf{1.155} \pm 0.008$ & $1.172 \pm 0.008$ & $1.196 \pm 0.008$ \\
    Lasso & $2.397 \pm 0.026$ & $\textbf{1.921} \pm 0.024$ & $1.933 \pm 0.024$ \\
    Linear Regression & $\textbf{1.001} \pm 0.002$ & $1.084 \pm 0.005$ & $1.110 \pm 0.006$ \\
    Random Forest & $\textbf{1.281} \pm 0.013$ & $1.296 \pm 0.013$ & $1.317 \pm 0.014$ \\
    Ridge & $\textbf{1.001} \pm 0.002$ & $1.084 \pm 0.005$ & $1.110 \pm 0.006$ \\
    SVR (RBF) & $\textbf{1.094} \pm 0.004$ & $1.152 \pm 0.006$ & $1.177 \pm 0.007$ \\
    \midrule
    average & $1.430$ & $\textbf{1.301}$ & $1.322$ \\
    \bottomrule
  \end{tabular}
  \end{center}
\end{table}

\begin{table}
  \caption{Nonlinear dataset ($\sigma=1$)}\label{tab:NL10K1}
  \medskip
  \begin{center}
  \begin{tabular}{lccc}
    \toprule
    & none & CVAR1 & CVAR10 \\
    \midrule
    Elastic Net & $2.417 \pm 0.021$ & $\textbf{1.763} \pm 0.007$ & $1.780 \pm 0.008$ \\
    Gradient Boosting & $\textbf{1.227} \pm 0.009$ & $1.241 \pm 0.008$ & $1.271 \pm 0.009$ \\
    Lasso & $2.695 \pm 0.024$ & $\textbf{2.212} \pm 0.020$ & $2.223 \pm 0.020$ \\
    Linear Regression & $\textbf{1.494} \pm 0.003$ & $1.539 \pm 0.005$ & $1.560 \pm 0.005$ \\
    Random Forest & $\textbf{1.417} \pm 0.015$ & $1.437 \pm 0.015$ & $1.462 \pm 0.015$ \\
    Ridge & $\textbf{1.494} \pm 0.003$ & $1.539 \pm 0.005$ & $1.560 \pm 0.005$ \\
    SVR (RBF) & $\textbf{1.202} \pm 0.004$ & $1.270 \pm 0.006$ & $1.298 \pm 0.007$ \\
    \midrule
    average & $1.707$ & $\textbf{1.572}$ & $1.594$ \\
    \bottomrule
  \end{tabular}
  \end{center}
\end{table}

\begin{table}
  \caption{Heteroscedastic noise dataset ($\sigma=1$)}\label{tab:HSN10K1}
  \medskip
  \begin{center}
  \begin{tabular}{lccc}
    \toprule
    & none & CVAR1 & CVAR10 \\
    \midrule
    Elastic Net & $2.315 \pm 0.022$ & $\textbf{1.722} \pm 0.008$ & $1.740 \pm 0.009$ \\
    Gradient Boosting & $\textbf{1.550} \pm 0.007$ & $1.554 \pm 0.007$ & $1.576 \pm 0.007$ \\
    Lasso & $2.604 \pm 0.025$ & $\textbf{2.173} \pm 0.021$ & $2.184 \pm 0.021$ \\
    Linear Regression & $\textbf{1.430} \pm 0.004$ & $1.481 \pm 0.005$ & $1.504 \pm 0.006$ \\
    Random Forest & $\textbf{1.640} \pm 0.010$ & $1.645 \pm 0.010$ & $1.665 \pm 0.011$ \\
    Ridge & $\textbf{1.430} \pm 0.004$ & $1.481 \pm 0.005$ & $1.504 \pm 0.006$ \\
    SVR (RBF) & $\textbf{1.509} \pm 0.005$ & $1.543 \pm 0.006$ & $1.564 \pm 0.007$ \\
    \midrule
    average & $1.782$ & $\textbf{1.657}$ & $1.677$ \\
    \bottomrule
  \end{tabular}
  \end{center}
\end{table}

\begin{table}
  \caption{Heavy-tailed noise dataset ($\sigma=1$)}\label{tab:HTN10K1}
  \medskip
  \begin{center}
  \begin{tabular}{lccc}
    \toprule
    & none & CVAR1 & CVAR10 \\
    \midrule
    Elastic Net & $2.512 \pm 0.029$ & $\textbf{1.958} \pm 0.025$ & $1.980 \pm 0.025$ \\
    Gradient Boosting & $1.819 \pm 0.026$ & $\textbf{1.814} \pm 0.026$ & $1.841 \pm 0.026$ \\
    Lasso & $2.785 \pm 0.030$ & $\textbf{2.380} \pm 0.028$ & $2.393 \pm 0.028$ \\
    Linear Regression & $\textbf{1.707} \pm 0.026$ & $1.744 \pm 0.026$ & $1.772 \pm 0.026$ \\
    Random Forest & $1.901 \pm 0.026$ & $\textbf{1.895} \pm 0.026$ & $1.920 \pm 0.026$ \\
    Ridge & $\textbf{1.707} \pm 0.026$ & $1.744 \pm 0.026$ & $1.772 \pm 0.026$ \\
    SVR (RBF) & $\textbf{1.770} \pm 0.025$ & $1.795 \pm 0.025$ & $1.820 \pm 0.025$ \\
    \midrule
    average & $2.028$ & $\textbf{1.904}$ & $1.928$ \\
    \bottomrule
  \end{tabular}
  \end{center}
\end{table}

\begin{table}
  \caption{Outlier contamination dataset ($\sigma=1$)}\label{tab:OC10K1}
  \medskip
  \begin{center}
  \begin{tabular}{lccc}
    \toprule
    & none & CVAR1 & CVAR10 \\
    \midrule
    Elastic Net & $2.330 \pm 0.022$ & $\textbf{1.719} \pm 0.014$ & $1.747 \pm 0.014$ \\
    Gradient Boosting & $1.561 \pm 0.013$ & $\textbf{1.555} \pm 0.014$ & $1.590 \pm 0.014$ \\
    Lasso & $2.620 \pm 0.025$ & $\textbf{2.182} \pm 0.022$ & $2.198 \pm 0.022$ \\
    Linear Regression & $\textbf{1.430} \pm 0.014$ & $1.476 \pm 0.014$ & $1.512 \pm 0.014$ \\
    Random Forest & $1.656 \pm 0.015$ & $\textbf{1.653} \pm 0.015$ & $1.683 \pm 0.015$ \\
    Ridge & $\textbf{1.430} \pm 0.014$ & $1.476 \pm 0.014$ & $1.512 \pm 0.014$ \\
    SVR (RBF) & $\textbf{1.500} \pm 0.014$ & $1.532 \pm 0.014$ & $1.564 \pm 0.014$ \\
    \midrule
    average & $1.790$ & $\textbf{1.656}$ & $1.687$ \\
    \bottomrule
  \end{tabular}
  \end{center}
\end{table}

\begin{table}
  \caption{Sparse signals dataset ($\sigma=1$)}\label{tab:SHDS10K1}
  \medskip
  \begin{center}
  \begin{tabular}{lccc}
    \toprule
    & none & CVAR1 & CVAR10 \\
    \midrule
    Elastic Net & $1.144 \pm 0.012$ & $\textbf{1.016} \pm 0.003$ & $1.019 \pm 0.003$ \\
    Gradient Boosting & $\textbf{1.004} \pm 0.002$ & $1.009 \pm 0.002$ & $1.011 \pm 0.002$ \\
    Lasso & $1.214 \pm 0.018$ & $\textbf{1.083} \pm 0.011$ & $1.086 \pm 0.011$ \\
    Linear Regression & $\textbf{0.998} \pm 0.002$ & $1.005 \pm 0.002$ & $1.008 \pm 0.002$ \\
    Random Forest & $1.020 \pm 0.002$ & $\textbf{1.018} \pm 0.002$ & $1.021 \pm 0.002$ \\
    Ridge & $\textbf{0.998} \pm 0.002$ & $1.005 \pm 0.002$ & $1.008 \pm 0.002$ \\
    SVR (RBF) & $1.030 \pm 0.002$ & $\textbf{1.027} \pm 0.003$ & $1.030 \pm 0.003$ \\
    \midrule
    average & $1.058$ & $\textbf{1.023}$ & $1.026$ \\
  \end{tabular}
  \end{center}
\end{table}

\begin{table}
  \caption{Covariate shift dataset ($\sigma=1$)}\label{tab:CS10K1}
  \medskip
  \begin{center}
  \begin{tabular}{lccc}
    \toprule
    & none & CVAR1 & CVAR10 \\
    \midrule
    Elastic Net & $2.619 \pm 0.068$ & $\textbf{1.796} \pm 0.040$ & $1.842 \pm 0.043$ \\
    Gradient Boosting & $\textbf{1.462} \pm 0.024$ & $1.569 \pm 0.037$ & $1.625 \pm 0.042$ \\
    Lasso & $3.047 \pm 0.083$ & $\textbf{2.542} \pm 0.063$ & $2.566 \pm 0.064$ \\
    Linear Regression & $\textbf{1.002} \pm 0.001$ & $1.313 \pm 0.036$ & $1.377 \pm 0.041$ \\
    Random Forest & $\textbf{1.778} \pm 0.036$ & $1.846 \pm 0.042$ & $1.892 \pm 0.046$ \\
    Ridge & $\textbf{1.002} \pm 0.001$ & $1.313 \pm 0.036$ & $1.377 \pm 0.041$ \\
    SVR (RBF) & $\textbf{1.888} \pm 0.052$ & $2.019 \pm 0.060$ & $2.053 \pm 0.062$ \\
    \midrule
    average & $1.828$ & $\textbf{1.771}$ & $1.819$ \\
    \bottomrule
  \end{tabular}
  \end{center}
\end{table}

\begin{table}
  \caption{Friedman 1 dataset ($\sigma=1$)}\label{tab:F1_10K1}
  \medskip
  \begin{center}
  \begin{tabular}{lccc}
    \toprule
    & none & CVAR1 & CVAR10 \\
    \midrule
    Elastic Net & $3.344 \pm 0.005$ & $\textbf{2.649} \pm 0.005$ & $2.658 \pm 0.005$ \\
    Gradient Boosting & $\textbf{1.307} \pm 0.004$ & $1.396 \pm 0.003$ & $1.425 \pm 0.003$ \\
    Lasso & $3.305 \pm 0.005$ & $\textbf{2.775} \pm 0.005$ & $2.783 \pm 0.005$ \\
    Linear Regression & $2.629 \pm 0.005$ & $\textbf{2.628} \pm 0.005$ & $2.637 \pm 0.005$ \\
    Random Forest & $\textbf{1.498} \pm 0.003$ & $1.582 \pm 0.003$ & $1.608 \pm 0.003$ \\
    Ridge & $2.629 \pm 0.005$ & $\textbf{2.628} \pm 0.005$ & $2.637 \pm 0.005$ \\
    SVR (RBF) & $\textbf{1.309} \pm 0.003$ & $1.427 \pm 0.003$ & $1.456 \pm 0.003$ \\
    \midrule
    average & $2.289$ & $\textbf{2.155}$ & $2.172$ \\
    \bottomrule
  \end{tabular}
  \end{center}
\end{table}

\begin{table}
  \caption{Friedman 2 dataset ($\sigma=1$)}\label{tab:F2_10K1}
  \medskip
  \begin{center}
  \begin{tabular}{lccc}
    \toprule
    & none & CVAR1 & CVAR10 \\
    \midrule
    Elastic Net & $181.762 \pm 0.340$ & $\textbf{82.630} \pm 0.173$ & $84.079 \pm 0.173$ \\
    Gradient Boosting & $\textbf{15.208} \pm 0.061$ & $48.217 \pm 0.120$ & $50.498 \pm 0.151$ \\
    Lasso & $137.802 \pm 0.233$ & $\textbf{82.620} \pm 0.171$ & $84.070 \pm 0.171$ \\
    Linear Regression & $137.783 \pm 0.233$ & $\textbf{82.663} \pm 0.172$ & $84.112 \pm 0.172$ \\
    Random Forest & $\textbf{5.791} \pm 0.021$ & $47.214 \pm 0.121$ & $49.544 \pm 0.153$ \\
    Ridge & $137.784 \pm 0.233$ & $\textbf{82.663} \pm 0.172$ & $84.112 \pm 0.172$ \\
    SVR (RBF) & $140.800 \pm 0.538$ & $\textbf{105.108} \pm 0.381$ & $105.258 \pm 0.404$ \\
    \midrule
    average & $108.133$ & $\textbf{75.874}$ & $77.382$ \\
    \bottomrule
   \end{tabular}
  \end{center}
\end{table}

\begin{table}
  \caption{Friedman 3 dataset ($\sigma=1$)}\label{tab:F3_10K1}
  \medskip
  \begin{center}
  \begin{tabular}{lccc}
    \toprule
    & none & CVAR1 & CVAR10 \\
    \midrule
    Elastic Net & $1.050 \pm 0.002$ & $1.050 \pm 0.002$ & $\textbf{1.050} \pm 0.002$ \\
    Gradient Boosting & $1.009 \pm 0.002$ & $\textbf{1.007} \pm 0.002$ & $1.008 \pm 0.002$ \\
    Lasso & $1.050 \pm 0.002$ & $1.050 \pm 0.002$ & $\textbf{1.050} \pm 0.002$ \\
    Linear Regression & $1.022 \pm 0.002$ & $1.014 \pm 0.002$ & $\textbf{1.013} \pm 0.002$ \\
    Random Forest & $1.042 \pm 0.002$ & $\textbf{1.018} \pm 0.002$ & $1.018 \pm 0.002$ \\
    Ridge & $1.022 \pm 0.002$ & $1.014 \pm 0.002$ & $\textbf{1.013} \pm 0.002$ \\
    SVR (RBF) & $1.012 \pm 0.002$ & $1.010 \pm 0.002$ & $\textbf{1.010} \pm 0.002$ \\
    \midrule
    average & $1.030$ & $\textbf{1.023}$ & $1.023$ \\
    \bottomrule
  \end{tabular}
  \end{center}
\end{table}

Tables~\ref{tab:BL1K3}--\ref{tab:F3_1K3} give results for $n=1000$ examples
and noise level $\sigma=3$.
For these smaller datasets the results are mixed,
as we said in the main paper.

\begin{table}
  \caption{Bounded logistic dataset ($n=1000$ and $\sigma=3$)}\label{tab:BL1K3}
  \medskip
  \begin{center}
  \begin{tabular}{lccc}
    \toprule
    & none & CVAR1 & CVAR10 \\
    \midrule
    Elastic Net & $3.831 \pm 0.021$ & $\textbf{3.303} \pm 0.018$ & $3.418 \pm 0.018$ \\
    Gradient Boosting & $3.420 \pm 0.019$ & $\textbf{3.375} \pm 0.016$ & $3.492 \pm 0.017$ \\
    Lasso & $4.035 \pm 0.023$ & $\textbf{3.687} \pm 0.023$ & $3.770 \pm 0.023$ \\
    Linear Regression & $3.288 \pm 0.019$ & $\textbf{3.119} \pm 0.015$ & $3.247 \pm 0.015$ \\
    Random Forest & $3.449 \pm 0.018$ & $\textbf{3.436} \pm 0.017$ & $3.543 \pm 0.018$ \\
    Ridge & $3.288 \pm 0.019$ & $\textbf{3.119} \pm 0.015$ & $3.247 \pm 0.015$ \\
    SVR (RBF) & $3.345 \pm 0.018$ & $\textbf{3.235} \pm 0.016$ & $3.353 \pm 0.016$ \\
    \midrule
    average & $3.522$ & $\textbf{3.325}$ & $3.439$ \\
    \bottomrule
  \end{tabular}
  \end{center}
\end{table}

\begin{table}
  \caption{Linear Gaussian dataset ($n=1000$ and $\sigma=3$)}\label{tab:LG1K3}
  \medskip
  \begin{center}
  \begin{tabular}{lccc}
    \toprule
    & none & CVAR1 & CVAR10 \\
    \midrule
    Elastic Net & $3.549 \pm 0.022$ & $\textbf{3.292} \pm 0.019$ & $3.431 \pm 0.023$ \\
    Gradient Boosting & $\textbf{3.228} \pm 0.016$ & $3.318 \pm 0.020$ & $3.459 \pm 0.024$ \\
    Lasso & $3.755 \pm 0.025$ & $\textbf{3.557} \pm 0.023$ & $3.655 \pm 0.027$ \\
    Linear Regression & $\textbf{3.016} \pm 0.015$ & $3.177 \pm 0.018$ & $3.333 \pm 0.022$ \\
    Random Forest & $\textbf{3.289} \pm 0.018$ & $3.358 \pm 0.022$ & $3.492 \pm 0.026$ \\
    Ridge & $\textbf{3.016} \pm 0.015$ & $3.177 \pm 0.018$ & $3.333 \pm 0.022$ \\
    SVR (RBF) & $\textbf{3.256} \pm 0.019$ & $3.306 \pm 0.021$ & $3.437 \pm 0.025$ \\
    \midrule
    average & $\textbf{3.301}$ & $3.312$ & $3.449$ \\
    \bottomrule
  \end{tabular}
  \end{center}
\end{table}

\begin{table}
  \caption{Nonlinear dataset ($n=1000$ and $\sigma=3$)}\label{tab:NL1K3}
  \medskip
  \begin{center}
  \begin{tabular}{lccc}
    \toprule
    & none & CVAR1 & CVAR10 \\
    \midrule
    Elastic Net & $3.749 \pm 0.023$ & $\textbf{3.476} \pm 0.020$ & $3.626 \pm 0.024$ \\
    Gradient Boosting & $\textbf{3.265} \pm 0.017$ & $3.391 \pm 0.020$ & $3.561 \pm 0.026$ \\
    Lasso & $3.941 \pm 0.026$ & $\textbf{3.727} \pm 0.024$ & $3.838 \pm 0.028$ \\
    Linear Regression & $\textbf{3.210} \pm 0.014$ & $3.371 \pm 0.018$ & $3.537 \pm 0.023$ \\
    Random Forest & $\textbf{3.365} \pm 0.019$ & $3.461 \pm 0.023$ & $3.619 \pm 0.028$ \\
    Ridge & $\textbf{3.209} \pm 0.014$ & $3.371 \pm 0.018$ & $3.537 \pm 0.023$ \\
    SVR (RBF) & $\textbf{3.365} \pm 0.019$ & $3.424 \pm 0.020$ & $3.577 \pm 0.025$ \\
    \midrule
    average & $\textbf{3.444}$ & $3.460$ & $3.614$ \\
    \bottomrule
  \end{tabular}
  \end{center}
\end{table}

\begin{table}
  \caption{Heteroscedastic noise dataset ($n=1000$ and $\sigma=3$)}\label{tab:HSN1K3}
  \medskip
  \begin{center}
  \begin{tabular}{lccc}
    \toprule
    & none & CVAR1 & CVAR10 \\
    \midrule
    Elastic Net & $4.677 \pm 0.032$ & $\textbf{4.475} \pm 0.031$ & $4.597 \pm 0.032$ \\
    Gradient Boosting & $4.545 \pm 0.031$ & $\textbf{4.545} \pm 0.031$ & $4.662 \pm 0.033$ \\
    Lasso & $4.837 \pm 0.033$ & $\textbf{4.680} \pm 0.034$ & $4.766 \pm 0.035$ \\
    Linear Regression & $\textbf{4.282} \pm 0.029$ & $4.388 \pm 0.030$ & $4.524 \pm 0.032$ \\
    Random Forest & $\textbf{4.522} \pm 0.031$ & $4.550 \pm 0.032$ & $4.662 \pm 0.034$ \\
    Ridge & $\textbf{4.282} \pm 0.029$ & $4.388 \pm 0.030$ & $4.524 \pm 0.032$ \\
    SVR (RBF) & $\textbf{4.479} \pm 0.031$ & $4.498 \pm 0.032$ & $4.608 \pm 0.033$ \\
    \midrule
    average & $4.518$ & $\textbf{4.503}$ & $4.620$ \\
    \bottomrule
  \end{tabular}
  \end{center}
\end{table}

\begin{table}
  \caption{Heavy-tailed noise dataset ($n=1000$ and $\sigma=3$)}\label{tab:HTN1K3}
  \medskip
  \begin{center}
  \begin{tabular}{lccc}
    \toprule
    & none & CVAR1 & CVAR10 \\
    \midrule
    Elastic Net & $5.266 \pm 0.091$ & $\textbf{5.097} \pm 0.092$ & $5.187 \pm 0.092$ \\
    Gradient Boosting & $5.297 \pm 0.095$ & $\textbf{5.172} \pm 0.093$ & $5.264 \pm 0.093$ \\
    Lasso & $5.409 \pm 0.089$ & $\textbf{5.278} \pm 0.091$ & $5.338 \pm 0.091$ \\
    Linear Regression & $\textbf{4.949} \pm 0.093$ & $5.029 \pm 0.093$ & $5.130 \pm 0.093$ \\
    Random Forest & $5.237 \pm 0.093$ & $\textbf{5.178} \pm 0.094$ & $5.266 \pm 0.093$ \\
    Ridge & $\textbf{4.949} \pm 0.093$ & $5.029 \pm 0.093$ & $5.130 \pm 0.093$ \\
    SVR (RBF) & $\textbf{5.084} \pm 0.093$ & $5.110 \pm 0.093$ & $5.190 \pm 0.093$ \\
    \midrule
    average & $5.170$ & $\textbf{5.128}$ & $5.215$ \\
    \bottomrule
  \end{tabular}
  \end{center}
\end{table}

\begin{table}
  \caption{Outlier contamination dataset ($n=1000$ and $\sigma=3$)}\label{tab:OC1K3}
  \medskip
  \begin{center}
  \begin{tabular}{lccc}
    \toprule
    & none & CVAR1 & CVAR10 \\
    \midrule
    Elastic Net & $4.561 \pm 0.102$ & $\textbf{4.362} \pm 0.106$ & $4.472 \pm 0.103$ \\
    Gradient Boosting & $4.488 \pm 0.108$ & $\textbf{4.413} \pm 0.106$ & $4.526 \pm 0.104$ \\
    Lasso & $4.727 \pm 0.099$ & $\textbf{4.580} \pm 0.103$ & $4.650 \pm 0.101$ \\
    Linear Regression & $\textbf{4.171} \pm 0.108$ & $4.277 \pm 0.107$ & $4.401 \pm 0.105$ \\
    Random Forest & $4.480 \pm 0.107$ & $\textbf{4.443} \pm 0.106$ & $4.542 \pm 0.103$ \\
    Ridge & $\textbf{4.171} \pm 0.108$ & $4.277 \pm 0.107$ & $4.401 \pm 0.105$ \\
    SVR (RBF) & $\textbf{4.329} \pm 0.106$ & $4.371 \pm 0.106$ & $4.471 \pm 0.104$ \\
    \midrule
    average & $4.418$ & $\textbf{4.389}$ & $4.495$ \\
    \bottomrule
  \end{tabular}
  \end{center}
\end{table}

\begin{table}
  \caption{Sparse signals dataset ($n=1000$ and $\sigma=3$)}\label{tab:SHDS1K3}
  \medskip
  \begin{center}
  \begin{tabular}{lccc}
    \toprule
    & none & CVAR1 & CVAR10 \\
    \midrule
    Elastic Net & $3.022 \pm 0.018$ & $\textbf{2.980} \pm 0.016$ & $2.996 \pm 0.017$ \\
    Gradient Boosting & $3.106 \pm 0.017$ & $\textbf{3.022} \pm 0.017$ & $3.031 \pm 0.018$ \\
    Lasso & $3.042 \pm 0.018$ & $\textbf{3.003} \pm 0.015$ & $3.021 \pm 0.017$ \\
    Linear Regression & $\textbf{2.978} \pm 0.016$ & $2.993 \pm 0.016$ & $3.003 \pm 0.018$ \\
    Random Forest & $3.064 \pm 0.017$ & $\textbf{3.018} \pm 0.017$ & $3.028 \pm 0.019$ \\
    Ridge & $\textbf{2.978} \pm 0.016$ & $2.993 \pm 0.016$ & $3.003 \pm 0.018$ \\
    SVR (RBF) & $3.041 \pm 0.017$ & $\textbf{3.024} \pm 0.017$ & $3.030 \pm 0.018$ \\
    \midrule
    average & $3.033$ & $\textbf{3.005}$ & $3.016$ \\
    \bottomrule
  \end{tabular}
  \end{center}
\end{table}

\begin{table}
  \caption{Covariate shift dataset ($n=1000$ and $\sigma=3$)}\label{tab:CS1K3}
  \medskip
  \begin{center}
  \begin{tabular}{lccc}
    \toprule
    & none & CVAR1 & CVAR10 \\
    \midrule
    Elastic Net & $3.954 \pm 0.064$ & $\textbf{3.711} \pm 0.068$ & $3.936 \pm 0.082$ \\
    Gradient Boosting & $\textbf{3.490} \pm 0.029$ & $3.809 \pm 0.071$ & $4.030 \pm 0.084$ \\
    Lasso & $4.273 \pm 0.080$ & $\textbf{4.077} \pm 0.073$ & $4.241 \pm 0.086$ \\
    Linear Regression & $\textbf{3.040} \pm 0.018$ & $3.551 \pm 0.068$ & $3.808 \pm 0.083$ \\
    Random Forest & $\textbf{3.624} \pm 0.045$ & $3.868 \pm 0.072$ & $4.086 \pm 0.085$ \\
    Ridge & $\textbf{3.040} \pm 0.018$ & $3.551 \pm 0.068$ & $3.808 \pm 0.083$ \\
    SVR (RBF) & $\textbf{4.151} \pm 0.086$ & $4.178 \pm 0.090$ & $4.342 \pm 0.100$ \\
    \midrule
    average & $\textbf{3.653}$ & $3.821$ & $4.036$ \\
    \bottomrule
  \end{tabular}
  \end{center}
\end{table}

\begin{table}
  \caption{Friedman 1 dataset ($n=1000$ and $\sigma=3$)}\label{tab:F1_1K3}
  \medskip
  \begin{center}
  \begin{tabular}{lccc}
    \toprule
    & none & CVAR1 & CVAR10 \\
    \midrule
    Elastic Net & $4.385 \pm 0.023$ & $\textbf{4.058} \pm 0.021$ & $4.270 \pm 0.022$ \\
    Gradient Boosting & $\textbf{3.378} \pm 0.018$ & $3.642 \pm 0.020$ & $3.920 \pm 0.022$ \\
    Lasso & $4.365 \pm 0.023$ & $\textbf{4.159} \pm 0.021$ & $4.359 \pm 0.022$ \\
    Linear Regression & $\textbf{3.893} \pm 0.021$ & $4.044 \pm 0.021$ & $4.256 \pm 0.022$ \\
    Random Forest & $\textbf{3.568} \pm 0.019$ & $3.783 \pm 0.021$ & $4.042 \pm 0.022$ \\
    Ridge & $\textbf{3.893} \pm 0.021$ & $4.044 \pm 0.021$ & $4.256 \pm 0.022$ \\
    SVR (RBF) & $\textbf{3.746} \pm 0.021$ & $3.865 \pm 0.021$ & $4.102 \pm 0.023$ \\
    \midrule
    average & $\textbf{3.890}$ & $3.942$ & $4.172$ \\
    \bottomrule
  \end{tabular}
  \end{center}
\end{table}

\begin{table}
  \caption{Friedman 2 dataset ($n=1000$ and $\sigma=3$)}\label{tab:F2_1K3}
  \medskip
  \begin{center}
  \begin{tabular}{lccc}
    \toprule
    & none & CVAR1 & CVAR10 \\
    \midrule
    Elastic Net & $181.546 \pm 1.247$ & $\textbf{151.839} \pm 1.044$ & $181.262 \pm 1.604$ \\
    Gradient Boosting & $\textbf{23.936} \pm 0.181$ & $141.982 \pm 0.952$ & $171.762 \pm 1.655$ \\
    Lasso & $\textbf{137.739} \pm 0.784$ & $151.822 \pm 1.039$ & $181.216 \pm 1.603$ \\
    Linear Regression & $\textbf{137.743} \pm 0.784$ & $151.853 \pm 1.040$ & $181.245 \pm 1.604$ \\
    Random Forest & $\textbf{20.166} \pm 0.202$ & $141.894 \pm 0.935$ & $171.675 \pm 1.642$ \\
    Ridge & $\textbf{137.746} \pm 0.784$ & $151.853 \pm 1.040$ & $181.246 \pm 1.604$ \\
    SVR (RBF) & $345.749 \pm 2.419$ & $\textbf{169.058} \pm 1.464$ & $191.746 \pm 2.032$ \\
    \midrule
    average & $\textbf{140.661}$ & $151.471$ & $180.022$ \\
    \bottomrule
  \end{tabular}
  \end{center}
\end{table}

\begin{table}
  \caption{Friedman 3 dataset ($n=1000$ and $\sigma=3$)}\label{tab:F3_1K3}
  \medskip
  \begin{center}
  \begin{tabular}{lccc}
    \toprule
    & none & CVAR1 & CVAR10 \\
    \midrule
    Elastic Net & $3.019 \pm 0.015$ & $\textbf{3.019} \pm 0.015$ & $3.020 \pm 0.015$ \\
    Gradient Boosting & $3.130 \pm 0.015$ & $3.030 \pm 0.015$ & $\textbf{3.021} \pm 0.015$ \\
    Lasso & $3.019 \pm 0.015$ & $\textbf{3.019} \pm 0.015$ & $3.020 \pm 0.015$ \\
    Linear Regression & $\textbf{3.018} \pm 0.015$ & $3.030 \pm 0.015$ & $3.018 \pm 0.015$ \\
    Random Forest & $3.146 \pm 0.016$ & $3.032 \pm 0.015$ & $\textbf{3.023} \pm 0.015$ \\
    Ridge & $\textbf{3.018} \pm 0.015$ & $3.030 \pm 0.015$ & $3.018 \pm 0.015$ \\
    SVR (RBF) & $3.054 \pm 0.016$ & $3.030 \pm 0.015$ & $\textbf{3.020} \pm 0.015$ \\
    \midrule
    average & $3.058$ & $3.027$ & $\textbf{3.020}$ \\
    \bottomrule
  \end{tabular}
  \end{center}
\end{table}

In Tables~\ref{tab:BL1K1}--\ref{tab:F3_1K1} we have $n=1000$ examples,
but the noise level is smaller, $\sigma=1$.
The results for our methods appear even worse
than for the more challenging case $\sigma=3$;
the base algorithms work best on average in nine cases out of eleven.

\begin{table}
  \caption{Bounded logistic dataset ($n=1000$ and $\sigma=1$)}\label{tab:BL1K1}
  \medskip
  \begin{center}
  \begin{tabular}{lccc}
    \toprule
    & none & CVAR1 & CVAR10 \\
    \midrule
    Elastic Net & $2.580 \pm 0.022$ & $\textbf{1.750}\pm0.016$ & $1.850 \pm 0.015$ \\
    Gradient Boosting & $\textbf{1.708}\pm0.021$ & $1.735 \pm 0.015$ & $1.841 \pm 0.014$ \\
    Lasso & $2.877 \pm 0.026$ & $\textbf{2.385}\pm0.026$ & $2.449 \pm 0.025$ \\
    Linear Regression & $1.630 \pm 0.023$ & $\textbf{1.380}\pm0.007$ & $1.506 \pm 0.007$ \\
    Random Forest & $\textbf{1.870}\pm0.023$ & $1.932 \pm 0.019$ & $2.024 \pm 0.018$ \\
    Ridge & $1.629 \pm 0.023$ & $\textbf{1.380}\pm0.007$ & $1.506 \pm 0.007$ \\
    SVR (RBF) & $1.578 \pm 0.017$ & $\textbf{1.499}\pm0.008$ & $1.614 \pm 0.008$ \\
    \midrule
    average & $1.982$ & $\textbf{1.723}$ & $1.827$ \\
    \bottomrule
  \end{tabular}
  \end{center}
\end{table}

\begin{table}
  \caption{Linear Gaussian dataset ($n=1000$ and $\sigma=1$)}
  \label{tab:LG1K1}
  \medskip
  \begin{center}
    \begin{tabular}{lccc}
    \toprule
    & none & CVAR1 & CVAR10 \\
    \midrule
    Elastic Net & $2.141 \pm 0.029$ & $\textbf{1.746} \pm 0.024$ & $1.957 \pm 0.031$ \\
    Gradient Boosting & $\textbf{1.326} \pm 0.015$ & $1.662 \pm 0.028$ & $1.890 \pm 0.035$ \\
    Lasso & $2.471 \pm 0.032$ & $\textbf{2.184} \pm 0.030$ & $2.329 \pm 0.035$ \\
    Linear Regression & $\textbf{1.005} \pm 0.005$ & $1.530 \pm 0.025$ & $1.773 \pm 0.032$ \\
    Random Forest & $\textbf{1.566} \pm 0.024$ & $1.787 \pm 0.032$ & $1.993 \pm 0.039$ \\
    Ridge & $\textbf{1.005} \pm 0.005$ & $1.530 \pm 0.025$ & $1.772 \pm 0.032$ \\
    SVR (RBF) & $\textbf{1.418} \pm 0.020$ & $1.653 \pm 0.028$ & $1.878 \pm 0.035$ \\
    \midrule
    average & $\textbf{1.562}$ & $1.728$ & $1.942$ \\
    \bottomrule
  \end{tabular}
  \end{center}
\end{table}

\begin{table}
  \caption{Nonlinear dataset ($n=1000$ and $\sigma=1$)}
  \label{tab:NL1K1}
  \medskip
  \begin{center}
  \begin{tabular}{lccc}
    \toprule
    & none & CVAR1 & CVAR10 \\
    \midrule
    Elastic Net & $2.469 \pm 0.028$ & $\textbf{2.073} \pm 0.024$ & $2.285 \pm 0.031$ \\
    Gradient Boosting & $\textbf{1.430} \pm 0.016$ & $1.820 \pm 0.029$ & $2.077 \pm 0.036$ \\
    Lasso & $2.755 \pm 0.032$ & $\textbf{2.454} \pm 0.030$ & $2.607 \pm 0.035$ \\
    Linear Regression & $\textbf{1.496} \pm 0.008$ & $1.900 \pm 0.023$ & $2.139 \pm 0.031$ \\
    Random Forest & $\textbf{1.735} \pm 0.025$ & $1.993 \pm 0.033$ & $2.221 \pm 0.039$ \\
    Ridge & $\textbf{1.496} \pm 0.008$ & $1.900 \pm 0.023$ & $2.139 \pm 0.031$ \\
    SVR (RBF) & $\textbf{1.622} \pm 0.020$ & $1.873 \pm 0.028$ & $2.116 \pm 0.035$ \\
    \midrule
    average & $\textbf{1.858}$ & $2.002$ & $2.226$ \\
    \bottomrule
  \end{tabular}
  \end{center}
\end{table}

\begin{table}
  \caption{Heteroscedastic noise dataset ($n=1000$ and $\sigma=1$)}
  \label{tab:HSN1K1}
  \medskip
  \begin{center}
  \begin{tabular}{lccc}
    \toprule
    & none & CVAR1 & CVAR10 \\
    \midrule
    Elastic Net & $2.379 \pm 0.027$ & $\textbf{2.017} \pm 0.023$ & $2.221 \pm 0.030$ \\
    Gradient Boosting & $\textbf{1.697} \pm 0.014$ & $1.964 \pm 0.026$ & $2.182 \pm 0.033$ \\
    Lasso & $2.683 \pm 0.030$ & $\textbf{2.417} \pm 0.029$ & $2.558 \pm 0.034$ \\
    Linear Regression & $\textbf{1.427} \pm 0.010$ & $1.829 \pm 0.023$ & $2.062 \pm 0.030$ \\
    Random Forest & $\textbf{1.875} \pm 0.022$ & $2.058 \pm 0.030$ & $2.258 \pm 0.036$ \\
    Ridge & $\textbf{1.427} \pm 0.010$ & $1.829 \pm 0.023$ & $2.062 \pm 0.030$ \\
    SVR (RBF) & $\textbf{1.777} \pm 0.019$ & $1.947 \pm 0.026$ & $2.162 \pm 0.033$ \\
    \midrule
    average & $\textbf{1.895}$ & $2.009$ & $2.215$ \\
    \bottomrule
  \end{tabular}
  \end{center}
\end{table}

\begin{table}
  \caption{Heavy-tailed noise dataset ($n=1000$ and $\sigma=1$)}
  \label{tab:HTN1K1}
  \medskip
  \begin{center}
  \begin{tabular}{lccc}
    \toprule
    & none & CVAR1 & CVAR10 \\
    \midrule
    Elastic Net & $2.492 \pm 0.034$ & $\textbf{2.147} \pm 0.033$ & $2.335 \pm 0.037$ \\
    Gradient Boosting & $\textbf{1.907} \pm 0.032$ & $2.108 \pm 0.037$ & $2.309 \pm 0.040$ \\
    Lasso & $2.783 \pm 0.037$ & $\textbf{2.529} \pm 0.036$ & $2.660 \pm 0.039$ \\
    Linear Regression & $\textbf{1.650} \pm 0.031$ & $1.977 \pm 0.035$ & $2.192 \pm 0.038$ \\
    Random Forest & $\textbf{2.043} \pm 0.034$ & $2.189 \pm 0.039$ & $2.376 \pm 0.042$ \\
    Ridge & $\textbf{1.650} \pm 0.031$ & $1.977 \pm 0.035$ & $2.192 \pm 0.038$ \\
    SVR (RBF) & $\textbf{1.932} \pm 0.034$ & $2.083 \pm 0.036$ & $2.279 \pm 0.040$ \\
    \midrule
    average & $\textbf{2.065}$ & $2.144$ & $2.335$ \\
    \bottomrule
  \end{tabular}
  \end{center}
\end{table}

\begin{table}
  \caption{Outlier contamination dataset ($n=1000$ and $\sigma=1$)}
  \label{tab:OC1K1}
  \medskip
  \begin{center}
  \begin{tabular}{lccc}
    \toprule
    & none & CVAR1 & CVAR10 \\
    \midrule
    Elastic Net & $2.329 \pm 0.032$ & $\textbf{1.960} \pm 0.033$ & $2.162 \pm 0.034$ \\
    Gradient Boosting & $\textbf{1.684} \pm 0.034$ & $1.915 \pm 0.036$ & $2.127 \pm 0.037$ \\
    Lasso & $2.634 \pm 0.034$ & $\textbf{2.371} \pm 0.034$ & $2.508 \pm 0.036$ \\
    Linear Regression & $\textbf{1.390} \pm 0.036$ & $1.774 \pm 0.035$ & $2.005 \pm 0.036$ \\
    Random Forest & $\textbf{1.852} \pm 0.035$ & $2.011 \pm 0.037$ & $2.208 \pm 0.039$ \\
    Ridge & $\textbf{1.390} \pm 0.036$ & $1.774 \pm 0.035$ & $2.005 \pm 0.036$ \\
    SVR (RBF) & $\textbf{1.707} \pm 0.034$ & $1.887 \pm 0.036$ & $2.097 \pm 0.037$ \\
    \midrule
    average & $\textbf{1.855}$ & $1.956$ & $2.159$ \\
    \bottomrule
  \end{tabular}
  \end{center}
\end{table}

\begin{table}
  \caption{Sparse signals dataset ($n=1000$ and $\sigma=1$)}
  \label{tab:SHDS1K1}
  \medskip
  \begin{center}
  \begin{tabular}{lccc}
    \toprule
    & none & CVAR1 & CVAR10 \\
    \midrule
    Elastic Net & $1.153 \pm 0.016$ & $\textbf{1.060} \pm 0.011$ & $1.103 \pm 0.016$ \\
    Gradient Boosting & $\textbf{1.039} \pm 0.006$ & $1.067 \pm 0.012$ & $1.109 \pm 0.017$ \\
    Lasso & $1.204 \pm 0.018$ & $\textbf{1.131} \pm 0.013$ & $1.165 \pm 0.016$ \\
    Linear Regression & $\textbf{0.993} \pm 0.005$ & $1.050 \pm 0.011$ & $1.093 \pm 0.016$ \\
    Random Forest & $\textbf{1.030} \pm 0.006$ & $1.067 \pm 0.012$ & $1.108 \pm 0.017$ \\
    Ridge & $\textbf{0.993} \pm 0.005$ & $1.050 \pm 0.011$ & $1.093 \pm 0.016$ \\
    SVR (RBF) & $\textbf{1.073} \pm 0.009$ & $1.090 \pm 0.014$ & $1.127 \pm 0.018$ \\
    \midrule
    average & $\textbf{1.069}$ & $1.073$ & $1.114$ \\
    \bottomrule
  \end{tabular}
  \end{center}
\end{table}

\begin{table}
  \caption{Covariate shift dataset ($n=1000$ and $\sigma=1$)}
  \label{tab:CS1K1}
  \medskip
  \begin{center}
  \begin{tabular}{lccc}
    \toprule
    & none & CVAR1 & CVAR10 \\
    \midrule
    Elastic Net & $2.706 \pm 0.080$ & $\textbf{2.383} \pm 0.088$ & $2.691 \pm 0.105$ \\
    Gradient Boosting & $\textbf{1.717} \pm 0.037$ & $2.384 \pm 0.091$ & $2.695 \pm 0.107$ \\
    Lasso & $3.149 \pm 0.095$ & $\textbf{2.932} \pm 0.090$ & $3.140 \pm 0.104$ \\
    Linear Regression & $\textbf{1.013} \pm 0.006$ & $2.132 \pm 0.092$ & $2.491 \pm 0.108$ \\
    Random Forest & $\textbf{2.152} \pm 0.060$ & $2.577 \pm 0.092$ & $2.852 \pm 0.108$ \\
    Ridge & $\textbf{1.013} \pm 0.006$ & $2.132 \pm 0.092$ & $2.491 \pm 0.108$ \\
    SVR (RBF) & $\textbf{2.740} \pm 0.101$ & $2.939 \pm 0.111$ & $3.127 \pm 0.121$ \\
    \midrule
    average & $\textbf{2.070}$ & $2.497$ & $2.784$ \\
    \bottomrule
  \end{tabular}
  \end{center}
\end{table}

\begin{table}
  \caption{Friedman 1 dataset ($n=1000$ and $\sigma=1$)}
  \label{tab:F1_1K1}
  \medskip
  \begin{center}
  \begin{tabular}{lccc}
    \toprule
    & none & CVAR1 & CVAR10 \\
    \midrule
    Elastic Net & $3.330 \pm 0.015$ & $\textbf{2.929} \pm 0.014$ & $3.177 \pm 0.015$ \\
    Gradient Boosting & $\textbf{1.551} \pm 0.009$ & $2.232 \pm 0.012$ & $2.575 \pm 0.015$ \\
    Lasso & $3.301 \pm 0.016$ & $\textbf{3.053} \pm 0.015$ & $3.291 \pm 0.016$ \\
    Linear Regression & $\textbf{2.634} \pm 0.014$ & $2.908 \pm 0.014$ & $3.156 \pm 0.015$ \\
    Random Forest & $\textbf{2.039} \pm 0.011$ & $2.509 \pm 0.012$ & $2.817 \pm 0.015$ \\
    Ridge & $\textbf{2.633} \pm 0.014$ & $2.908 \pm 0.014$ & $3.156 \pm 0.015$ \\
    SVR (RBF) & $\textbf{2.172} \pm 0.013$ & $2.532 \pm 0.014$ & $2.831 \pm 0.015$ \\
    \midrule
    average & $\textbf{2.523}$ & $2.725$ & $3.000$ \\
    \bottomrule
  \end{tabular}
  \end{center}
\end{table}

\begin{table}
  \caption{Friedman 2 dataset ($n=1000$ and $\sigma=1$)}
  \label{tab:F2_1K1}
  \medskip
  \begin{center}
  \begin{tabular}{lccc}
    \toprule
    & none & CVAR1 & CVAR10 \\
    \midrule
    Elastic Net & $181.524 \pm 1.246$ & $\textbf{151.829} \pm 1.043$ & $181.229 \pm 1.603$ \\
    Gradient Boosting & $\textbf{23.714} \pm 0.181$ & $141.966 \pm 0.956$ & $171.717 \pm 1.659$ \\
    Lasso & $\textbf{137.715} \pm 0.782$ & $151.815 \pm 1.038$ & $181.180 \pm 1.603$ \\
    Linear Regression & $\textbf{137.720} \pm 0.782$ & $151.840 \pm 1.039$ & $181.208 \pm 1.603$ \\
    Random Forest & $\textbf{19.889} \pm 0.206$ & $141.886 \pm 0.935$ & $171.629 \pm 1.642$ \\
    Ridge & $\textbf{137.723} \pm 0.782$ & $151.839 \pm 1.039$ & $181.208 \pm 1.603$ \\
    SVR (RBF) & $345.751 \pm 2.424$ & $\textbf{169.015} \pm 1.462$ & $191.694 \pm 2.032$ \\
    \midrule
    average & $\textbf{140.576}$ & $151.455$ & $179.981$ \\
    \bottomrule
  \end{tabular}
  \end{center}
\end{table}

\begin{table}
  \caption{Friedman 3 dataset ($n=1000$ and $\sigma=1$)}
  \label{tab:F3_1K1}
  \medskip
  \begin{center}
  \begin{tabular}{lccc}
    \toprule
    & none & CVAR1 & CVAR10 \\
    \midrule
    Elastic Net & $1.048 \pm 0.005$ & $1.048 \pm 0.005$ & $\textbf{1.048} \pm 0.005$ \\
    Gradient Boosting & $1.045 \pm 0.005$ & $\textbf{1.024} \pm 0.005$ & $1.027 \pm 0.005$ \\
    Lasso & $1.048 \pm 0.005$ & $1.048 \pm 0.005$ & $\textbf{1.048} \pm 0.005$ \\
    Linear Regression & $1.024 \pm 0.005$ & $\textbf{1.021} \pm 0.005$ & $1.021 \pm 0.005$ \\
    Random Forest & $1.054 \pm 0.005$ & $\textbf{1.027} \pm 0.005$ & $1.028 \pm 0.005$ \\
    Ridge & $1.024 \pm 0.005$ & $\textbf{1.021} \pm 0.005$ & $1.021 \pm 0.005$ \\
    SVR (RBF) & $1.035 \pm 0.005$ & $\textbf{1.024} \pm 0.005$ & $1.026 \pm 0.005$ \\
    \midrule
    average & $1.040$ & $\textbf{1.030}$ & $1.031$ \\
    \bottomrule
  \end{tabular}
  \end{center}
\end{table}

Finally, in Tables~\ref{tab:CB}--\ref{tab:SP} we show results
for the following four real-life datasets:
Bias correction
(in full, Bias correction of numerical prediction model temperature forecast)
\citep{BC} with 7750 examples and 21 features,
Wine quality \citep{WQ}
with 6497 examples and 12 features
(including colour, red or white),
Airfoil self-noise \citep{AF} with 1503 examples and 5 features,
and Student performance \citep{SP} with 2161 examples and 39 features.
For the last dataset, we define the label as the cumulative test score
computed by summing the reading and mathematics scaled scores
across all four grades (Kindergarten through Grade 3) for each student.
The results are mixed; in two cases,
our methods slightly improve the performance of the base algorithms on average.
(One of the datasets where our methods do not improve, and even slightly lower,
the quality of predictions is the Wine quality dataset in Table~\ref{tab:WQ};
for this dataset the label variable is bounded
but with bounds that are far from typical values of the labels.)

\section{Code Availability}

Python code for reproducing our experiments
in Sect.~\ref{sec:experiments} and Appendix~\ref{app:extra}
is available at the GitHub repository
\href{https://github.com/ip200/ivar-experiments}{https://github.com/ip200/ivar-experiments}
\citep{Petej:2026}.

\begin{table}
  \caption{Bias correction dataset}\label{tab:CB}
  \medskip
  \begin{center}
  \begin{tabular}{lccc}
    \toprule
    & none & CVAR1 & CVAR10 \\
    \midrule
    Elastic Net & $1.836 \pm 0.003$ & $\textbf{1.590} \pm 0.003$ & $1.605 \pm 0.003$ \\
    Gradient Boosting & $\textbf{1.208} \pm 0.002$ & $1.266 \pm 0.003$ & $1.286 \pm 0.003$ \\
    Lasso & $1.977 \pm 0.004$ & $\textbf{1.738} \pm 0.003$ & $1.751 \pm 0.003$ \\
    Linear Regression & $\textbf{1.463} \pm 0.003$ & $1.507 \pm 0.003$ & $1.523 \pm 0.003$ \\
    Random Forest & $\textbf{0.979} \pm 0.002$ & $1.067 \pm 0.003$ & $1.092 \pm 0.003$ \\
    Ridge & $\textbf{1.463} \pm 0.003$ & $1.507 \pm 0.003$ & $1.523 \pm 0.003$ \\
    SVR (RBF) & $\textbf{1.152} \pm 0.003$ & $1.233 \pm 0.003$ & $1.254 \pm 0.003$ \\
    \midrule
    average & $1.440$ & $\textbf{1.416}$ & $1.434$ \\
    \bottomrule
  \end{tabular}
  \end{center}
\end{table}

\begin{table}
  \caption{Wine quality dataset}\label{tab:WQ}
  \medskip
  \begin{center}
  \begin{tabular}{lccc}
    \toprule
    & none & CVAR1 & CVAR10 \\
    \midrule
    Elastic Net & $\textbf{0.870} \pm 0.001$ & $0.870 \pm 0.001$ & $0.870 \pm 0.001$ \\
    Gradient Boosting & $0.681 \pm 0.001$ & $\textbf{0.681} \pm 0.001$ & $0.682 \pm 0.001$ \\
    Lasso & $\textbf{0.870} \pm 0.001$ & $0.870 \pm 0.001$ & $0.870 \pm 0.001$ \\
    Linear Regression & $0.733 \pm 0.001$ & $\textbf{0.732} \pm 0.001$ & $0.732 \pm 0.001$ \\
    Random Forest & $\textbf{0.604} \pm 0.002$ & $0.611 \pm 0.001$ & $0.612 \pm 0.001$ \\
    Ridge & $0.733 \pm 0.001$ & $\textbf{0.732} \pm 0.001$ & $0.732 \pm 0.001$ \\
    SVR (RBF) & $0.676 \pm 0.001$ & $\textbf{0.676} \pm 0.001$ & $0.676 \pm 0.001$ \\
    \midrule
    average & $\textbf{0.738}$ & $0.739$ & $0.739$ \\
    \bottomrule
  \end{tabular}
  \end{center}
\end{table}

\begin{table}
  \caption{Airfoil self-noise dataset}
  \label{tab:AF}
  \medskip
  \begin{center}
  \begin{tabular}{lccc}
    \toprule
    & none & CVAR1 & CVAR10 \\
    \midrule
    Elastic Net & $5.616 \pm 0.018$ & $\textbf{4.857} \pm 0.020$ & $4.988 \pm 0.019$ \\
    Gradient Boosting & $\textbf{2.658} \pm 0.016$ & $3.122 \pm 0.016$ & $3.402 \pm 0.017$ \\
    Lasso & $5.541 \pm 0.019$ & $\textbf{5.061} \pm 0.022$ & $5.176 \pm 0.021$ \\
    Linear Regression & $4.820 \pm 0.021$ & $\textbf{4.667} \pm 0.018$ & $4.813 \pm 0.018$ \\
    Random Forest & $\textbf{1.767} \pm 0.014$ & $2.605 \pm 0.015$ & $2.940 \pm 0.016$ \\
    Ridge & $4.820 \pm 0.020$ & $\textbf{4.668} \pm 0.018$ & $4.813 \pm 0.018$ \\
    SVR (RBF) & $\textbf{3.797} \pm 0.019$ & $4.093 \pm 0.018$ & $4.293 \pm 0.018$ \\
    \midrule
    average & $\textbf{4.145}$ & $4.153$ & $4.347$ \\
    \bottomrule
  \end{tabular}
  \end{center}
\end{table}

\begin{table}
  \caption{Student performance dataset}
  \label{tab:SP}
  \medskip
  \begin{center}
  \begin{tabular}{lccc}
    \toprule
    & none & CVAR1 & CVAR10 \\
    \midrule
    Elastic Net & $\textbf{241.009}\pm0.626$ & $241.361 \pm 0.627$ & $241.593 \pm 0.635$ \\
    Gradient Boosting & $231.824 \pm 0.613$ & $\textbf{231.133}\pm0.623$ & $231.877 \pm 0.631$ \\
    Lasso & $\textbf{240.262}\pm0.633$ & $240.708 \pm 0.634$ & $241.015 \pm 0.640$ \\
    Linear Regression & $241.279 \pm 0.649$ & $\textbf{241.061}\pm0.632$ & $241.335 \pm 0.639$ \\
    Random Forest & $235.259 \pm 0.691$ & $\textbf{233.089}\pm0.668$ & $233.800 \pm 0.673$ \\
    Ridge & $\textbf{240.821}\pm0.636$ & $240.930 \pm 0.630$ & $241.216 \pm 0.636$ \\
    SVR (RBF) & $256.059 \pm 0.745$ & $\textbf{244.047}\pm0.636$ & $244.192 \pm 0.646$ \\
    \midrule
    average & $240.931$ & $\textbf{238.904}$ & $239.290$ \\
    \bottomrule
  \end{tabular}
  \end{center}
\end{table}
\end{document}